\documentclass{article}

\usepackage{graphicx} 
\usepackage{subfigure}
\usepackage{wrapfig}
\usepackage{xspace}
\usepackage{caption}

\usepackage{rotating}

\usepackage{algorithm}
\usepackage{algorithmic}

\usepackage{enumitem}
\usepackage{hyperref}

\usepackage{fullpage}
\usepackage{authblk}

\usepackage{Definitions}
\usepackage{color}

\renewcommand{\leq}{\leqslant}
\renewcommand{\le}{\leqslant}
\renewcommand{\geq}{\geqslant}
\renewcommand{\ge}{\geqslant}

\title{\huge Scalable Kernel Methods via Doubly Stochastic Gradients}

\author{
    Bo Dai$^1$, Bo Xie$^1$, Niao He$^2$, Yingyu Liang$^1$, Anant Raj$^1$,
    Maria-Florina Balcan$^3$, Le Song$^1$\\
    $^1$ College of Computing, Georgia Institute of Technology\\
    \{bodai, bxie33, yliang39, araj34\}@gatech.edu, lsong@cc.gatech.edu\\
    $^2$ School of Industrial \& Systems Engineering, Georgia Institute of Technology\\
    nhe6@gatech.edu\\
    $^3$ School of Computer Science, Carnegie Mellon University\\
    ninamf@cs.cmu.edu
}

\begin{document}
\maketitle

\begin{abstract}
  The general perception is that kernel methods are not scalable, and neural nets are the methods of choice for large-scale nonlinear learning problems. Or have we simply not tried hard enough for kernel methods? Here we propose an approach that scales up kernel methods using a novel concept called ``\emph{doubly stochastic functional gradients}''. Our approach relies on the fact that many kernel methods can be expressed as convex optimization problems, and we solve the problems by making \emph{two unbiased} stochastic approximations to the functional gradient, one using random training points and another using random features associated with the kernel, and then descending using this noisy functional gradient. Our algorithm is simple, does \emph{not} need to commit to a preset number of random features, and allows the flexibility of the function class to grow as we see more incoming data in the streaming setting. We show that a function learned by this procedure after $t$ iterations converges to the optimal function in the reproducing kernel Hilbert space in rate $O(1/t)$, and achieves a generalization performance of $O(1/\sqrt{t})$. Our approach can readily scale kernel methods up to the regimes which are dominated by neural nets. We show that our method can achieve competitive performance to neural nets in datasets such as 2.3 million energy materials from MolecularSpace, 8 million handwritten digits from MNIST, and 1 million photos from ImageNet using convolution features.
\end{abstract}

\section{Introduction}\label{sec:intro}

\setlength{\abovedisplayskip}{3pt}
\setlength{\abovedisplayshortskip}{1pt}
\setlength{\belowdisplayskip}{3pt}
\setlength{\belowdisplayshortskip}{1pt}
\setlength{\jot}{2pt}

\setlength{\floatsep}{2ex}
\setlength{\textfloatsep}{2ex}

The general perception is that kernel methods are not scalable. When it comes to large-scale nonlinear learning problems, the methods of choice so far are neural nets where theoretical understanding remains incomplete. Are kernel methods really not scalable? Or is it simply because we have not tried hard enough, while neural nets have exploited sophisticated design of feature architectures, virtual example generation for dealing with invariance, stochastic gradient descent for efficient training, and GPUs for further speedup?

A bottleneck in scaling up kernel methods is the storage and computation of the kernel matrix, $K$, which is usually dense. Storing the matrix  requires $O(n^2)$ space, and computing it takes $O(n^2 d)$ operations, where $n$ is the number of data points and $d$ is the dimension. There have been many great attempts to scale up kernel methods, including efforts from numerical linear algebra, functional analysis, and numerical optimization perspectives.

A common numerical linear algebra approach is to approximate the kernel matrix using low-rank factors, $K\approx A^\top A$, with $A \in \RR^{r \times n}$ and rank $r \leqslant n$. This low-rank approximation usually requires $O(nr^2+nrd)$ operations, and then subsequent kernel algorithms can directly operate on $A$. Many works, such as Greedy basis selection techniques~\cite{SmoSch00}, Nystr{\"o}m approximation~\cite{WilSee00b} and incomplete Cholesky decomposition~\cite{FinSch01}, all followed this strategy. In practice, one observes that kernel methods with approximated kernel matrices often result in a few percentage of losses in performance.
In fact, without further assumption on the regularity of the kernel matrix, the generalization ability after low-rank approximation is typically of the order $O(1/\sqrt{r} + 1/\sqrt{n})$~\cite{DriMah05,CorMohTal10}, which implies that the rank needs to be nearly linear in the number of data points! Thus, in order for kernel methods to achieve the best generalization ability, the low-rank approximation based approaches quickly become impractical for big datasets due to their $O(n^3+n^2d)$ preprocessing time and $O(n^2)$ memory requirement.

Random feature approximation is another popular approach for scaling up kernel methods~\cite{RahRec08,LeSarSmo13}. Instead of approximating the kernel matrix, the method directly approximates the kernel function using explicit feature maps. The advantage of this approach is that the random feature matrix for $n$ data points can be computed in time $O(nrd)$ using $O(nr)$ memory, where $r$ is the number of random features. Subsequent algorithms then only operate on an $O(nr)$ matrix. Similar to low-rank kernel matrix approximation approach, the generalization ability of random feature approach is of the order $O(1/\sqrt{r} + 1/\sqrt{n})$~\cite{RahRec09,LopSraSmoGhaetal14}, which implies that the number of random features also needs to be $O(n)$. Another common drawback of these two approaches is that it is not easy to adapt the solution from a small $r$ to a large $r'$. Often one is interested in increasing the kernel matrix approximation rank or the number of random features to obtain a better generalization 
ability. Then special procedures need to be designed to reuse the solution obtained from a small $r$, which is not straightforward.

Another approach that addresses the scalability issue rises from optimization perspective. One general strategy is to solve the dual forms of kernel methods using coordinate or block-coordinate descent~(\eg,~\cite{Platt98, Joachims99,ShaTon13b}). By doing so, each iteration of the algorithm only incurs $O(nrd)$ computation and $O(nr)$ memory, where $r$ is the size of the parameter block. A second strategy is to perform functional gradient descent by looking at a batch of data points at a time~(\eg,~\cite{KivSmoWil04b, RatBag07}). Thus, the computation and memory requirements are also $O(nrd)$ and $O(nr)$ respectively in each iteration, where $r$ is the batch size. These approaches can easily change to a different $r$ without restarting the optimization and has no loss  in generalization ability since they do not approximate the kernel matrix or function. However, a serious drawback of these approaches is that, without further approximation, all support vectors need to be kept for testing, which can be as
big as the entire training set! (\eg, kernel ridge regression and non-separable nonlinear classification problems.)

In summary, there exists a delicate trade-off between computation, memory and statistics if one wants to scale up kernel methods. Inspired by various previous efforts, we propose a simple yet general strategy to scale up many kernel methods using a novel concept called ``\emph{doubly stochastic functional gradients}''. Our method relies on the fact that most kernel methods can be expressed as convex optimization problems over functions in reproducing kernel Hilbert spaces (RKHS) and solved via functional gradient descent. Our algorithm proceeds by making \emph{two unbiased} stochastic approximations to the functional gradient, one using random training points and the other one using random features associated with the kernel, and then descending using this noisy functional gradient.
The key intuitions behind our algorithm originate from 
\begin{itemize}
  \item[(i)] the property of stochastic gradient descent algorithm that as long as the stochastic gradient is unbiased, the convergence of the algorithm is guaranteed~\cite{NemJudLanSha09}; and 
  \item[(ii)] the property of pseudo-random number generators that the random samples can in fact  be completely determined by an initial value (a seed). 
\end{itemize}
We exploit these properties and enable kernel methods to achieve  better balances between computation, memory and statistics. Our method interestingly combines  kernel methods, functional analysis, stochastic optimization and algorithmic trick, and it possesses a number of desiderata:\\[-3mm]

\noindent  {\bf Generality and simplicity.} Our approach applies to many kernel methods, such as kernel ridge regression, support vector machines, logistic regression, two-sample test, and many different types of kernels, such as shift-invariant kernels, polynomial kernels, general inner product kernels, and so on. The algorithm can be summarized in just a few lines of code (Algorithm 1 and 2). For a different problem and kernel, we just need to adapt the loss function and the random feature generator.\\[-3mm]

\noindent {\bf Flexibility.} Different from previous uses of random features which typically prefix the number of features and then optimize over the feature weightings, our approach allows the number of random features, and hence the flexibility of the function class, to grow with the number of data points. This allows our method to be applicable to data streaming setting, which is not possible for previous random feature approach, and achieve the full potential of nonparametric methods.\\[-3mm] 

\noindent {\bf Efficient computation.} The key computation of our method is evaluating the doubly stochastic functional gradient, which involves the generation of the random features with specific random seeds and the evaluation of these random features on the small batch of data points. For iteration $t$, the computational complexity is $O(td)$.\\[-3mm]

\noindent {\bf Small memory.} The doubly stochasticity also allows us to avoid keeping the support vectors which becomes prohibitive in large-scale streaming setting. Instead, we just need to keep a small program for regenerating the random features, and sample previously used random feature according to pre-specified random seeds. For iteration $t$, the memory needed is $O(t)$ independent of the dimension of the data.\\[-3mm] 

\noindent {\bf Theoretical guarantees.} We provide a novel and nontrivial analysis involving Hilbert space martingale and a newly proved recurrence relation, and show that the estimator produced by our algorithm, which might be outside of the RKHS, converges to the optimal RKHS function. More specifically, both in expectation and with high probability, our algorithm can estimate the optimal function in the RKHS in the rate of $O(1/t)$, which are indeed optimal~\cite{NemJudLanSha09}, and achieve a generalization bound of $O(1/\sqrt{t})$. The variance of the random features, introduced during our second approximation to the functional gradient, only contributes additively to the constant in the final convergence rate. These results are the first of the kind in kernel method literature, which can be of independent interest.\\[-3mm] 

\noindent {\bf Strong empirical performance.} Our algorithm can readily scale kernel methods up to the regimes which are previously dominated by neural nets. We show that our method compares favorably to other scalable kernel methods in medium scale datasets, and to neural nets in big datasets such as 8 million handwritten digits from MNIST, 2.3 million materials from MolecularSpace, and 1 million photos from ImageNet using convolution features. Our results suggest that kernel methods, theoretically well-grounded methods, can potentially replace neural nets in many large scale real-world problems where nonparametric estimation are needed.\\[-3mm]

In the remainder, we will first introduce preliminaries on kernel methods and functional gradients. We will then describe our algorithm and provide both theoretical and empirical supports.

\section{Duality between Kernels and Random Processes} \label{sec:kernel}

Kernel methods owe their name to the use of kernel functions, $k(x,x'):\Xcal \times \Xcal \mapsto \RR$, which are symmetric positive definite (PD), meaning that for all $n > 1$, and $x_1,\ldots,x_n \in \Xcal$, and $c_1,\ldots,c_n \in \RR$, we have $\sum_{i,j=1}^n c_i c_j k(x_i, x_j) \geqslant 0$. There is an intriguing duality between kernels and stochastic processes which  will play a crucial role in our later algorithm design. More specifically, 
\begin{theorem}[\eg,\cite{Devinatz53}; \cite{HeiBou04}]
  If $k(x,x')$ is a PD kernel, then there exists a set $\Omega$, a measure $\PP$ on $\Omega$, and random feature $\phi_{\omega}(x):\Xcal\mapsto\RR$ from $L_2(\Omega,\PP)$, such that
  $
    k(x, x') = \int_{\Omega}\, \phi_{\omega}(x)\, \phi_{\omega}(x')\, d \PP(\omega).
  $
\end{theorem}
Essentially, the above integral representation relates the kernel function to a random process $\omega$ with measure $\PP(\omega)$. Note that the integral representation may not be unique. For instance, the random process can be a Gaussian process on $\Xcal$ with the sample function $\phi_{\omega}(x)$, and $k(x,x')$ is simply the covariance function between two point $x$ and $x'$. If the kernel is also continuous and shift invariant,~\ie,~$k(x,x')= k(x-x')$ for $x \in \RR^d$, then the integral representation specializes into a form characterized by inverse Fourier transformation (\eg,~\cite[Theorem 6.6]{Wendland05}),\\[-5mm]
\begin{theorem}[Bochner]
  A continuous, real-valued, symmetric and shift-invariant function $k(x-x')$ on $\RR^d$ is a PD kernel if and only if there is a finite non-negative measure $\PP(\omega)$ on $\RR^d$, such that
  $
    k(x-x') = \int_{\RR^d} \, e^{i \omega^\top (x-x')}\, d\PP(\omega) = \int_{\RR^d \times [0,2\pi]} 2 \, \cos(\omega^\top x + b)\, \cos(\omega^\top x' + b)\, d \rbr{\PP(\omega) \times \PP(b)},
  $
  where $\PP(b)$ is a uniform distribution on $[0,2\pi]$, and $\phi_{\omega}(x) = \sqrt{2}\cos(\omega^\top x + b)$.\\[-5mm]
\end{theorem}
For Gaussian RBF kernel, $k(x-x')=\exp(-\|x - x'\|^2/2\sigma^2)$, this yields a Gaussian distribution $\PP(\omega)$ with density proportional to $\exp(-\sigma^2\|\omega\|^2/2)$; for the Laplace kernel, this yields a Cauchy distribution; and for the Martern kernel, this yields the convolutions of the unit ball~\cite{SchSmo02}. 

Similar representation where the explicit form of $\phi_{\omega}(x)$ and $\PP(\omega)$ are known can also be derived for rotation invariant kernel, $k(x,x') = k(\inner{x}{x'})$, using Fourier transformation on sphere~\cite{SchSmo02}. For polynomial kernels, $k(x,x')=(\inner{x}{x'}+c)^p$, a random tensor sketching approach can also be used~\cite{PhaPag13}. Explicit random features have been designed for many other kernels, such as dot product kernel~\cite{KarKar12}, additive/multiplicative class of homogeneous kernels~\cite{VedZis12}, \eg, Hellinger's, $\chi^2$, Jensen-Shannon's and Intersection kernel, as well as kernels on Abelian semigroups~\cite{YanSinFanAvretal14}. We summarized these kernels with their explicit features and associated densities in Table~\ref{table:explicit_features}. 

{
\begin{sidewaystable}
\caption{Summary of kernels in~\cite{RahRec08, RasWil06, KarKar12, PhaPag13, VedZis12, YanSinFanAvretal14, ChoSaul09} and their explicit features}\label{table:explicit_features}
    \begin{tabular}{ll|c|c|c}
      \hline
      &Kernel &$k(x, x')$ &$\phi_{\omega}(x)$ &$p(\omega)$\\
      \hline
      &Gaussian &$\exp(-\frac{\|x - x'\|_2^2}{2})$ &$\exp(-i\omega^\top x)$ &${(2\pi)}^{-\frac{d}{2}}\exp(-\frac{\|\omega\|_2^2}{2})$  \\ 
      &Laplacian &$\exp(-{\|x - x'\|_1})$ &$\exp(-i\omega^\top x)$ &$\prod_{i=1}^d \frac{1}{\pi(1+\omega_i^2)}$ \\ 
      &Cauchy   &$\prod_{i=1}^d\frac{2}{1+(x_i -x'_i )^2}$ &$\exp(-i\omega^\top x)$ & $\exp(-{\|\omega\|_1})$ \\ 
      &Mat{\'e}rn &$\frac{2^{1-\nu}}{\Gamma(\nu)}\bigg(\frac{\sqrt{2\nu}\|x - x'\|_2}{\ell}\bigg)^\nu\hspace{-2mm} K_\nu\bigg(\frac{\sqrt{2\nu}\|x - x'\|_2}{\ell}\bigg)$ &$\exp(-i\omega^\top x)$ &$h(\nu, d,\ell)\bigg(\frac{2\nu}{\ell^2}\hspace{-1mm} + \hspace{-1mm}4\pi^2\|\omega\|_2^2\bigg)^{\nu + d/2}$ \\
      &Dot Product &$\sum_{n=0}^\infty a_n\langle x, x'\rangle^n\quad a_n\ge 0$ &$\sqrt{a_N p^{N+1}}\prod_{i=1}^N\omega_i^\top x$ &$\PP[N=n] = \frac{1}{p^{n+1}}$\\
      & & & &$p(\omega^j_i|N=n) = \frac{1}{2}^{\frac{\omega^j_i+1}{2}}\frac{1}{2}^{\frac{1-\omega^j_i}{2}}$\\
      &Polynomial &$(\langle x, x'\rangle + c)^p$ &$\mathtt{FFT}^{-1}(\mathtt{FFT}(C_1x)\odot\ldots\odot\mathtt{FFT}(C_px))$ &$C_j=S_jD_j$\\
      & & & & $D_j\in\RR^{d\times d}\, S_j\in\RR^{D\times d}$\\
      &Hellinger &$\sum_{i=1}^d\sqrt{x_i x'_i}$ &$2\omega^\top\sqrt{x}$ &$\frac{1}{2}^{\frac{\omega_i+1}{2}}\frac{1}{2}^{\frac{1-\omega_i}{2}},\,\, \omega_i\in \{-1, +1\}$\\
      &$\chi^2$ &$2\sum_{i=1}^d\frac{x_i x_i'}{x_i + x'_i}$ &$\big[\exp(-i\omega\log x_j)\sqrt{x_j}\big]_{j=1}^d$ &$ sech(\pi\omega)$ \\
      &Intersection &$\sum_{i=1}^d \min(x_i, x_i')$ &$\big[\exp(-i\omega\log x_j)\sqrt{2x_j}\big]_{j=1}^d$ &$\frac{1}{\pi(1 + 4\omega^2)}$ \\
      &Jensen-Shannon &$\sum_{i=1}^d K_{JS}(x_i, x'_i)$ &$\big[\exp(-i\omega\log x_j)\sqrt{2x_j}\big]_{j=1}^d$ &$\frac{sech(\pi \omega)}{\log 4(1 + 4\omega^2)}$ \\
      &\footnotesize{Skewed-$\chi^2$} &$2\prod_{i=1}^d\frac{\sqrt{x_i+c}\sqrt{x'_i+c}}{x_i + x'_i + 2c}$ &$\exp(-i\omega^\top \log(x+c))$ &$\prod_{i=1}^d sech(\pi \omega_i)$\\
      &\footnotesize{Skewed-Intersection} &$\prod_{i=1}^d\min\bigg(\sqrt{\frac{x_i+c}{x'_i+c}}, \sqrt{\frac{x'_i+c}{x_i+c}}\bigg)$ &$\exp(-i\omega^\top \log(x+c))$ &$\prod_{i=1}^d \frac{1}{\pi(1+4\omega_i^2)}$\\
      &\footnotesize{Exponential-Semigroup} &$\exp(-\beta \sum_{i=1}^d \sqrt{x_i + x'_j})$ &$\exp(-\omega^\top x)$ &$\prod_{i=1}^d\frac{\beta}{2\sqrt{\pi}}\omega_i^{-\frac{3}{2}}\exp(-\frac{\beta}{4\omega_i})$ \\
      &\footnotesize{Reciprocal-Semigroup} &$\prod_{i=1}^d\frac{\lambda}{x_i + x'_i + \lambda}$ &$\exp(-\omega^\top x)$ &$\prod_{i=1}^d\lambda\exp(-\lambda\omega_i)$ \\
      &Arc-Cosine &$\frac{1}{\pi}\|x\|^n\|x'\|^n J_n(\theta) $ &$(\omega^\top x)^n\max(0, \omega^\top x)$ &${2\pi}^{-\frac{d}{2}}\exp(-\frac{\|\omega\|_2^2}{2})$\\
      \hline
      \vspace{0.01in}
    \end{tabular}
    \centering
    \text{$D_j$ is random $\{\pm 1\}$ diagonal matrix and the columns of $S_j$ are uniformly selected from $\{e_1,\ldots, e_D\}$. $\nu$ and $\ell$ are positive parameters.}\\
    \text{$h(\nu, d, \ell) = \frac{2^d\pi^{d/2}\Gamma(\nu + d/2)(2\nu)^\nu}{\Gamma(\nu)\ell^{2\nu}}$. $K_\nu$ is a modified Bessel function. 
    $K_{JS}(x, x') = \frac{x}{2}\log_2\frac{x+x'}{x} + \frac{x'}{2}\log_2\frac{x+x'}{x'}$.}\\
    \text{$\theta = \cos^{-1}\frac{x^\top x'}{\|x\|\|x'\|}$, $J_n(\theta) = (-1)^n(\sin\theta)^{2n+1}\bigg(\frac{1}{\sin\theta}\frac{\partial}{\partial\theta}\bigg)^n\bigg(\frac{\pi - \theta}{\sin\theta}\bigg)$}
\end{sidewaystable}
}

Instead of finding the random process $\PP(\omega)$ and function $\phi_{\omega}(x)$ given a kernel, one can go the reverse direction, and construct kernels from random processes and functions (\eg, \cite{Wendland05}).
\begin{theorem}\label{thm:inverse_dual}
  If $k(x, x') = \int_{\Omega} \phi_{\omega}(x)^\top \phi_{\omega}(x')\, d \PP(\omega)$ for a nonnegative measure $\PP(\omega)$ on $\Omega$ and $\phi_{\omega}(x):\Xcal\mapsto \RR^r$, each component from $L_2(\Omega,\PP)$, then $k(x,x')$ is a PD kernel.
\end{theorem}
For instance, $\phi_{\omega}(x):=\cos(\omega^\top \psi_\theta(x) + b)$, where $\psi_\theta(x)$ can be a random convolution of the input $x$ parametrized by $\theta$, or $\phi_{\omega}(x) = [\phi_{\omega_1}(x), \phi_{\omega_2}(x),\ldots, \phi_{\omega_r}(x)]$, where $\phi_{\omega_1}(x)$ denote the random feature for kernel $k_1(x, x')$. The former random features define a hierachical kernel~\cite{ChoSaul09}, and the latter random features induce a linear combination of multiple kernels. It is worth to note that the Hellinger's, $\chi^2$, Jensen-Shannon's and Intersection kernels in~\cite{VedZis12} are special cases of multiple kernels combination. For simplicity, we assume $\phi_w(x)\in \RR$ following, and our algorithm is still applicable to $\phi_w(x)\in \RR^r$.

Another important concept is the reproducing kernel Hilbert space (RKHS). An RKHS $\Hcal$ on $\Xcal$ is a Hilbert space of functions from $\Xcal$ to $\RR$. $\Hcal$ is an RKHS if and only if there exists a $k(x,x'):\Xcal\times \Xcal \mapsto \RR$ such that
  $
    \forall x \in \Xcal, k(x,\cdot) \in \Hcal,~\text{and}~
    \forall f \in \Hcal, \inner{f(\cdot)}{k(x,\cdot)}_{\Hcal} = f(x).
  $
If such a $k(x,x')$ exist, it is unique and it is a PD kernel.
A function $f \in \Hcal$ if and only if $\nbr{f}_{\Hcal}^2 := \inner{f}{f}_{\Hcal} < \infty$, and its $L_2$ norm is dominated by RKHS norm
$
  \nbr{f}_{L_2} \leqslant \nbr{f}_{\Hcal}.
$

\section{Doubly Stochastic Functional Gradients}\label{sec:doubly_functional_sgd}
Many kernel methods can be written as convex optimizations over functions in the RKHS and solved using the functional gradient methods~\cite{KivSmoWil04b, RatBag07}. Inspired by these previous works, we will introduce a novel concept called ``\emph{doubly stochastic functional gradients}'' to address the scalability issue. Let $l(u,y)$ be a scalar (potentially non-smooth) loss function convex of $u \in \RR$. Let the subgradient of $l(u, y)$ with respect to $u$ be $l'(u, y)$. Given a PD kernel $k(x,x')$ and the associated RKHS $\Hcal$, many kernel methods try to find a function $f_* \in \Hcal$ which solves the optimization problem
\begin{align}\label{eq:primal}
  \argmin_{f\in \Hcal}~~ R(f) :=  \EE_{(x, y)}[l(f(x), y)] + \frac{\nu}{2}\nbr{f}_{\Hcal}^2 \quad \Longleftrightarrow\quad \argmin_{\nbr{f}_{\Hcal}\leqslant B(\nu)}~~ \EE_{(x, y)}[l(f(x), y)]
\end{align}
where $\nu > 0$ is a regularization parameter, $B(\nu)$ is a non-increasing function of $\nu$, and the data $(x,y)$ follow a distribution $\PP(x,y)$. The functional gradient $\nabla R(f)$ is defined as the linear term in the change of the objective after we perturb $f$ by $\epsilon$ in the direction of $g$,~\ie,
\begin{align}
  R(f + \epsilon g) = R(f) + \epsilon \inner{\nabla R(f)}{g}_{\Hcal} + O(\epsilon^2).
\end{align}
For instance, applying the above definition, we have $\nabla f(x) = \nabla \inner{f}{k(x,\cdot)}_{\Hcal} = k(x,\cdot)$, and $\nabla \nbr{f}_{\Hcal}^2 = \nabla \inner{f}{f}_{\Hcal} = 2f$.

{\bf Stochastic functional gradient.} Given a data point $(x,y)\sim \PP(x,y)$ and $f \in \Hcal$, the stochastic functional gradient of $\EE_{(x,y)}[l(f(x),y)]$ with respect to $f \in \Hcal$ is
\begin{align}
  \xi(\cdot):= l'(f(x), y)k(x,\cdot),
\end{align}
which is essentially a single data point approximation to the true functional gradient. Furthermore, for any $g \in \Hcal$, we have $\inner{\xi(\cdot)}{g}_{\Hcal} = l'(f(x), y)g(x)$. Inspired by the duality between kernel functions and random processes, we can make an additional approximation to the stochastic functional gradient using a random feature $\phi_{\omega}(x)$ sampled according to $\PP(\omega)$. More specifically,

{\bf Doubly stochastic functional gradient.} 
Let $\omega \sim \PP(\omega)$, then the doubly stochastic gradient of $\EE_{(x,y)}[l(f(x),y)]$ with respect to $f \in \Hcal$ is
\begin{align}
  \zeta(\cdot):= l'(f(x), y) \phi_{\omega}(x)\phi_{\omega}(\cdot).
\end{align}

\begin{wrapfigure}{r}{0.5\textwidth}
\vspace{-10mm}
\begin{center}
  \includegraphics[width=0.3\textwidth]{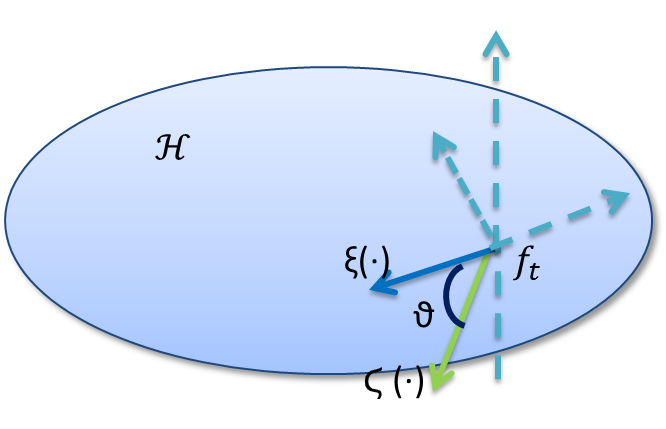}
\end{center}
\caption{$\vartheta$ is the angle betwen $\xi(\cdot)$ and $\zeta(\cdot)$, $\zeta(\cdot)$ may be outside of $\Hcal$.}
\end{wrapfigure}

Note that the stochastic functional gradient $\xi(\cdot)$ is in RKHS $\Hcal$ but $\zeta(\cdot)$ may be outside $\Hcal$, since $\phi_{\omega}(\cdot)$ may be outside the RKHS. For instance, for the Gaussian RBF kernel, the random feature $\phi_{\omega}(x) = \sqrt{2} \cos(\omega^\top x + b)$ is outside the RKHS associated with the kernel function.

However, these functional gradients are related by $\xi(\cdot) = \EE_{\omega} \sbr{\zeta(\cdot)}$, which lead to unbiased estimators of the original functional gradient,~\ie,
\begin{align}
  \nabla R(f) = \EE_{(x,y)}\sbr{\xi(\cdot)} + vf(\cdot),\\
  \text{and}\quad \nabla R(f) = \EE_{(x,y)} \EE_{\omega}\sbr{\zeta(\cdot)}  + vf(\cdot).
\end{align}
We emphasize that the source of randomness associated with the random feature is not present in the data, but artificially introduced by us. This is crucial for the development of our scalable algorithm in the next section. Meanwhile, it also creates additional challenges in the analysis of the algorithm which we will deal with carefully.

\section{Doubly Stochastic Kernel Machines}\label{sec:doubly_sgd}

\begin{figure}[t]
\begin{minipage}{0.50\textwidth}
  \hrule\vspace{1mm}
  \text{\bf Algorithm 1: $\cbr{\alpha_i}_{i=1}^t = \text{Train}(\PP(x,y))$}\vspace{1mm}
  \hrule\vspace{1mm}
  \text{\bf Require: $\PP(\omega),\, \phi_{\omega}(x),\,  l(f(x),y),\, \nu.$}\\[-4mm]
  \begin{algorithmic}[1]  \label{alg:ksup}
    \FOR{$i=1,\ldots, t$}
      \STATE Sample $(x_i, y_i) \sim \PP(x,y)$.
      \STATE Sample $\omega_i \sim \PP(\omega)$ with {\color{red}seed $i$}.
      \STATE $f(x_i) = \text{\bf Predict}(x_i,\cbr{\alpha_j}_{j=1}^{i-1})$.
      \STATE $\alpha_i = - \gamma_i l'(f(x_i), y_i) \phi_{\omega_i}(x_i)$.
      \STATE $\alpha_j = (1 - \gamma_i\nu) \alpha_j$ for $j =1,\ldots,i-1$.
    \ENDFOR
  \end{algorithmic}
  \hrule
\end{minipage}
~~~~~~~~
\begin{minipage}{0.45\textwidth}
  \hrule\vspace{1mm}
  \text{\bf Algorithm 2: $f(x)=\text{Predict}(x,\,\cbr{\alpha_i}_{i=1}^t)$}\vspace{1mm}
  \hrule\vspace{1mm}
  \text{\bf Require: $\PP(\omega),\, \phi_{\omega}(x).$}\\[-4mm]
  \begin{algorithmic}[1] \label{alg:testing}
    \STATE Set $f(x) = 0$.
    \FOR{$i=1,\ldots, t$}
      \STATE Sample $\omega_i \sim \PP(\omega)$ with {\color{red}seed $i$}.
      \STATE $f(x) = f(x) + \alpha_i \phi_{\omega_i}(x)$.
    \ENDFOR
  \end{algorithmic}
  \hrule
\end{minipage}
\end{figure}

The first key intuition behind our algorithm originates from the property of stochastic gradient descent algorithm that as long as the stochastic gradient is unbiased, the convergence of the algorithm is guaranteed~\cite{NemJudLanSha09}. In our algorithm, we will exploit this property and introduce \emph{two} sources of randomness, one from data and another artificial, to scale up kernel methods.

The second key intuition behind our algorithm is that the random features used in the doubly stochastic functional gradients will be sampled according to \emph{pseudo-random number generators}, where the sequences of apparently random samples can in fact  be completely determined by an initial value (a seed). Although these random samples are not the ``true'' random sample in the purest sense of the word, however they suffice for our task in practice.

More specifically, our algorithm proceeds by making two unbiased stochastic approximation to the functional gradient in each iteration, and then descending using this noisy functional gradient. The overall algorithms for training and prediction is summarized in Algorithm 1 and 2. The training algorithm essentially just performs random feature sampling and doubly stochastic gradient evaluation, and maintains a collection of real number $\cbr{\alpha_i}$, which is computationally efficient and memory friendly. A crucial step in the algorithm is to sample the random features with ``\emph{seed $i$}''. The seeds have to be aligned between training and prediction, and with the corresponding $\alpha_i$ obtained from each iteration. The learning rate $\gamma_t$ in the algorithm needs to be chosen as $O(1/t)$, as shown by our later analysis to achieve the best rate of convergence. For now, we assume that we have access to the data generating distribution $\PP(x,y)$. This can be modified to sample uniformly randomly
from a fixed dataset, without affecting the algorithm and the later convergence analysis. Let the sampled data and random feature parameters be $\Dcal^t:=\cbr{(x_i,y_i)}_{i=1}^t$ and $\omegab^t := \cbr{\omega_i}_{i=1}^t$ respectively after $t$ iteration, the function obtained by Algorithm 1 is a simple additive form of the doubly stochastic functional gradients
\begin{align}
  f_{t+1}(\cdot) = f_{t}(\cdot) - \gamma_t ( \zeta_t(\cdot) + \nu f_t(\cdot) ) = \sum\nolimits_{i=1}^t a_t^i \zeta_i(\cdot),\quad \forall t > 1, \quad\text{and}\quad  f_1(\cdot) = 0,
\end{align}
where $a_t^i = - \gamma_i \prod_{j=i+1}^t (1-\gamma_j \nu)$ are deterministic values depending on the step sizes $\gamma_j(i\leq j\leq t)$ and regularization parameter $\nu$. This simple form makes it easy for us to analyze its convergence.

We note that our algorithm can also take a mini-batch of points and random features at each step, and estimate an empirical covariance for preconditioning to achieve potentially better performance.

Our algorithm is general and can be applied to most of the kernel machines which are formulated in the convex optimization~(\ref{eq:primal}) in a RKHS $\Hcal$ associated with given kernel $k(x, x')$. We will instantiate the doubly stochastic gradients algorithms for a few commonly used kernel machines for different tasks and loss functions, \eg, regression, classification, quantile regression, novelty detection and estimating divergence functionals/likelihood ratio. Interestingly, the Gaussian process regression, which is a Bayesian model, can also be reformulated as the solution to particular convex optimizations in RKHS, and therefore, be approximated by the proposed algorithm. \\[-3mm]

\noindent{\bf Kernel Support Vector Machine~(SVM).} Hinge loss is used in kernel SVM where $l(u, y) = \max\{0, 1- uy\}$ with $y\in\{-1, 1\}$. We have $l'(u, y) =  \begin{cases} 0 & \,\text{if } yu \ge 1\\ -y & \, \text{if } yu <1 \\ \end{cases}$ and the step 5 in Algorithm.~1. becomes
$$
\alpha_i = \begin{cases} 0 & \,\text{if } y_if(x_i) \ge 1\\ \gamma_i y_i\phi_{\omega_i}(x_i) & \, \text{if } y_if(x_i)<1 \end{cases} .
$$

{\bf Remark:} \cite{KeeDeC05} used squared hinge loss, $l(u, y) = \frac{1}{2}\max\{0, 1-uy\}^2$, in $\ell_2$-SVM. With this loss function, we have $l'(u, y) = \begin{cases} 0 & \,\text{if } yu \ge 1\\ u-y & \, \text{if } yu <1 \\ \end{cases}$ and the step 5 in Algorithm.~1. becomes
$$
\alpha_i = \begin{cases} 0 & \,\text{if } y_if(x_i) \ge 1\\ \gamma_i (y_i - f(x_i))\phi_{\omega_i}(x_i) & \, \text{if } y_if(x_i)<1 \end{cases} .
$$

\noindent{\bf Kernel Logistic Regression. } Log loss is used in kernel logistic regression for binary classification where $l(u, y) = \log(1 + \exp(-yu))$ with $y\in \{-1,1\}$. We have $l'(u, y) = -\frac{y\exp(-yu)}{1 + \exp(-yu)}$ and the step 5 in Algorithm.~1. becomes
$$
\alpha_i = \gamma_iy_i\frac{\exp(-y_if(x_i))}{1 + \exp(-y_if(x_i))}\phi_{\omega_i}(x_i).
$$

For the multi-class kernel logistic regression, the $l(u, y) = -\sum_{{c}=1}^{C} \delta_c(y)u_c + \log\bigg(\sum_{{c}=1}^{C} \exp(u_c)\bigg)$ where $C$ is the number of categories, $u\in \RR^{C\times 1}$, $y\in \{1,\ldots,C\}$ and $\delta_c(y) = 1$ only if $y=c$, otherwise $\delta_c(y) = 0$. In such scenario, we denote $\fb(x_i) = [f^{1}(x_i),\ldots, f^C(x_i)]$, and therefore, the corresponding $\boldsymbol{\alpha} = [\alpha^{1},\ldots, \alpha^{C}]$. The update rule for $\boldsymbol{\alpha}$ in Algorithm.~1. is 

\begin{eqnarray*}
\alpha_i^{c} &=& \gamma_i\bigg(\delta_c(y_i) - \frac{\exp(f^{c}(x_i))}{\sum_{c=1}^{C}\exp(f^{c}(x_i))}\bigg)\phi_{\omega_i}(x_i) \quad  \forall c = 1,\ldots, C,\\
\alpha_j^{c} &=& (1 - \gamma_i\nu)\alpha_j^{c}\quad, \forall j<i, \forall c = 1,\ldots, C.
\end{eqnarray*}

\noindent{\bf Kernel Ridge Regression.} Square loss is used in kernel ridge regression where $l(u, y) = \frac{1}{2}(u -y)^2$. We have $l'(u, y) = (u - y)$ and the step 5 in Algorithm.~1. becomes
$$
\alpha_i = -\gamma_i(f(x_i) - y_i)\phi_{\omega_i}(x_i).
$$

\noindent{\bf Kernel Robust Regression.} Huber's loss is used for robust regression~\cite{MulSmoRatSchetal97} where 
$$
l(u, y) = \begin{cases} \frac{1}{2}(u - y)^2 &\text{if } |u-y|\le 1\\ |u - y| - \frac{1}{2} &\text{if } |u-y|>1\end{cases}.
$$
We have $l'(u, y) = \begin{cases} (u - y) &\text{if } |u-y|\le 1\\ \sgn(u - y) &\text{if } |u-y|>1\end{cases}$ and the step 5 in Algorithm.~1. becomes
$$
\alpha_i = \begin{cases} -\gamma_i(f(x_i) - y_i)\phi_{\omega_i}(x_i) & \text{if } |f(x_i) - y_i|\le 1\\ -\gamma_i\sgn(f(x_i) - y_i)\phi_{\omega_i}(x_i) &\text{if } |f(x_i) - y_i|>1 \end{cases}
$$

\noindent{\bf Kernel Support Vector Regression~(SVR).} $\epsilon$-insensitive loss function is used in kernel SVR where $l(u, y) = \max\{0, |u - y|-\epsilon\}$. We have $l'(u, y) = \begin{cases} 0 & \text{if } |u - y|\le \epsilon\\ \sgn(u - y)&\text{if } |u - y|>\epsilon \end{cases}$ and the step 5 in Algorithm.~1. becomes
$$
\alpha_i = \begin{cases} 0 & \text{if } |f(x_i) - y_i|\le \epsilon\\ -\gamma_i\sgn(f(x_i) - y_i)\phi_{\omega_i}(x_i) &\text{if } |f(x_i) - y_i|>\epsilon \end{cases}
$$

{\bf Remark:} Note that if we set $\epsilon = 0$, the $\epsilon$-intensitive loss function will become absolute deviatin, \ie, $l(u, y) = |u - y|$. Therefore, we have the updates for {\bf kernel least absolute deviatin regression}.

\noindent{\bf Kernel Quantile Regression.} The loss function for quantile regression is $l(u, y) = \max\{\tau(y - u), (1 - \tau)(u - y)\}$. We have $l'(u, y) = \begin{cases} 1-\tau & \,\text{if } u\ge y \\ -\tau & \, \text{if } u< y \\ \end{cases}$ and the step 5 in Algorithm.~1. becomes
$$
\alpha_i = \begin{cases}  \gamma_i(\tau - 1)\phi_{\omega_i}(x_i)& \,\text{if } f(x_i)\ge y_i \\ \gamma_i \tau\phi_{\omega_i}(x_i) & \, \text{if } f(x_i)<y_i \end{cases} .
$$

\noindent{\bf Kernel Novelty Detection.} The loss function $l(u, \tau) = \max\{0, \tau-u\}$~\cite{SchPlaShaSmoetal01} is proposed for novelty detection. Since $\tau$ is also a variable which needs to be optimized, the optimization problem is formulated as
\begin{eqnarray*}
\min_{\tau\in \RR, f\in \Hcal} \EE_{x}[l(f(x), \tau)] + \frac{\nu}{2}\|f\|_{\Hcal}^2 - \nu \tau,
\end{eqnarray*}
and the gradient of $l(u, \tau)$ is 
\begin{eqnarray*}
\frac{\partial l(u, \tau)}{\partial u} = \begin{cases}0 &\, \text{if } u\ge \tau \\ -1 &\, \text{if } u< \tau\end{cases},\quad  \frac{\partial l(u, \tau)}{\partial \tau} = \begin{cases}0 &\, \text{if } u\ge \tau \\ 1 &\, \text{if } u< \tau\end{cases}.
\end{eqnarray*}

The step 5 in Algorithm.~1. becomes
\begin{eqnarray*}
\alpha_i = \begin{cases}  0& \,\text{if } f(x_i)\ge \tau_{i-1} \\ \gamma_i \phi_{\omega_i}(x_i) & \, \text{if } f(x_i)<\tau_{i-1} \end{cases}, \quad
\tau_i = \begin{cases}  \tau_{i-1}+ \gamma_i\nu& \,\text{if } f(x_i)\ge \tau_{i-1} \\ \tau_{i-1} - \gamma_i(1 - \nu)& \, \text{if } f(x_i)<\tau_{i-1} \end{cases}.
\end{eqnarray*}

\noindent{\bf Kernel Density Ratio Estimation.} Based on the variational form of Ali-Silvey divergence, \ie, $\EE_{p}\big[r(\frac{q}{p})\big]$, where $r: \RR^+ \rightarrow \RR$ is a convex function with $r(1) = 0$, \cite{NguWaiJor08} proposed a nonparametric estimator for the logarithm of the density ratio, $\log\frac{q}{p}$, which is the solution of following convex optimization,
\begin{eqnarray}\label{eq:density_ratio}
\argmin_{f\in \Hcal}~~ \EE_{q}[\exp(f)] + \EE_{p}[r^*(-\exp(f))] + \frac{\nu}{2}\|f\|_\Hcal^2
\end{eqnarray}
where $r^*$ denotes the Fenchel-Legendre dual of $r$, $r(\tau) :=\sup_\chi \chi\tau - r^*(\chi)$. In Kullback-Leibler~(KL) divergence, the $r_{KL}(\tau) = -\log(\tau)$. Its Fenchel-Legendre dual is
\begin{align*}
r^*_{KL}(\tau)=\begin{cases}\infty &\text{if } \tau\ge 0\\ -1-\log(-\tau) &\text{if }\tau<0\end{cases}
\end{align*}
Specifically, the optimization becomes 
\begin{eqnarray*}
\min_{f\in \Hcal}~~ R(f) &=& \EE_{y\sim q}[\exp(f(y))] - \EE_{x\sim p}[f(x)] + \frac{\nu}{2}\|f\|_\Hcal^2 \\
&=&  2\EE_{z, x, y}\bigg[\delta_1(z)\exp(f(y)) - \delta_{0}(z)f(x)\bigg] + \frac{\nu}{2}\|f\|_\Hcal^2 .
\end{eqnarray*}
where $z\sim \text{Bernoulli}(0.5)$. Denote $l(u_x, u_y, z) = \delta_1(z)\exp(u_y) - \delta_{0}(z)u_x$, we have
\begin{eqnarray*}
l'(u_x, u_y, z) = \delta_1(z)\exp(u_y) - \delta_{0}(z)
\end{eqnarray*}
and the the step 5 in Algorithm.~1. becomes
$$
\alpha_i = -2\gamma_i(\delta_1(z_i)\exp(f(y_i))\phi_{\omega_i}(y_i) - \delta_{0}(z_i)\phi_{\omega_i}(x_i)),\quad z_i\sim \text{Bernoulli}(0.5).
$$
In particular, the $x_i$ and $y_i$ are not sampled in pair, they are sampled independently from $\PP(x)$ and $\QQ(x)$ respectively.

\cite{NguWaiJor08} proposed another convex optimization based on $r_{KL}(\tau)$ whose solution is a nonparametric estimator for the density ratio. \cite{SmoSonTeo09} designed $r_{nv}(\tau) = \max(0, \rho - \log\tau)$ for novelty detection. Similarly, the doubly stochastic gradients algorithm is also applicable to these loss functions.

\noindent{\bf Gaussian Process Regression.} The doubly stochastic gradients can be used for approximating the posterior of Gaussian process regression by reformulating the mean and variance of the predictive distribution as the solutions to the convex optimizations with particular loss functions. Let $y = f(x) + \epsilon$ where $\epsilon\sim \Ncal(0, \sigma^2)$ and $f(x)\sim \Gcal \Pcal(0, k(x, x'))$, given the dataset $\{x_i, y_i\}_{i=1}^n$, the posterior distribution of the function at the test point $x_*$ can be derived as
\begin{eqnarray}\label{eq:gpr_posterior}
f^*| X, \yb, x^* \sim \Ncal \rbr{ {k^*}^\top \rbr{K + \sigma^2 I}^{-1}\yb, \; k(x^*, x^*) - {k^*}^\top \rbr{K + \sigma^2 I}^{-1}k^* }
\end{eqnarray}
where $K\in \RR^{n\times n}$, $K_{ij} = K(x_i, x_j)$, $k^* = [k(x^*, x_1),\ldots, k(x^*, x_n)]^\top$ and $I\in \RR^{n\times n}$ is the identity matrix.

Obviously, the posterior mean of the Gaussian process for regression can be thought as the solution to optimization problem~(\ref{eq:primal}) with square loss and setting $\nu = 2\sigma^2$. Therefore, the update rule for approximating the posterior mean will be the same as kernel ridge regression.

To compute the predictive variance, we need to evaluate the ${k^*}^\top \rbr{K + \sigma^2 I}^{-1}k^*$. Following, we will introduce two different optimizations whose solutions can be used for evaluating the quantity.
\begin{enumerate}
\item Denote $\phi = [k(x_1, \cdot), \ldots, k(x_n, \cdot)]$, then 
\begin{eqnarray*}
{k^*}^\top \rbr{K + \sigma^2 I}^{-1}k^* &=& k(x^*,\cdot)^\top \phi \rbr{\phi^\top\phi + \sigma^2 I}^{-1}\phi^\top k(x^*, \cdot) \\
&=& k(x^*,\cdot)^\top \phi\phi^\top \rbr{\phi\phi^\top + \sigma^2 I}^{-1} k(x^*, \cdot)
\end{eqnarray*}
where the second equation based on identity $\rbr{\phi\phi^\top + \sigma^2I}\phi = \phi\rbr{\phi^\top\phi + \sigma^2 I}$. Therefore, we just need to estimate the operator:
\begin{eqnarray}\label{eq:variance_operator}
\Acal = \Ccal \rbr{\Ccal + \frac{\sigma^2}{n} I }^{-1}\quad \text{where}\quad \Ccal = \frac{1}{n}\phi\phi^\top = \frac{1}{n}\sum_{i=1}^n k(x_i, \cdot)\otimes k(x_i, \cdot).
\end{eqnarray}

We can express $\Acal$ as the solution to the following convex optimization problem
\begin{eqnarray*}
\min_{\Acal} R(\Acal) = \frac{1}{2n}\sum_{i=1}^n \nbr{k(x_i, \cdot) - \Acal k(x_i, \cdot)}_{\Hcal}^2 + \frac{\sigma^2}{2n}\nbr{\Acal}_{HS}^2
\end{eqnarray*}
where $\|\cdot\|_{HS}$ is the Hilbert-Schmidt norm of the operator. We can achieve the optimum by $\nabla R = 0$, which is equivalent to Eq.~\ref{eq:variance_operator}.

Based on this optimization, we approximate the $\Acal_t$ using $\sum_{i\le j, i=1}^t\theta_{ij}\phi_{\omega_i}(\cdot)\otimes\phi_{\omega_j}(\cdot)$ by doubly stochastic functional gradients. The update rule for $\theta$ is
\begin{eqnarray*}
\theta_{ij} &=& \bigg(1 - \frac{\sigma^2}{n}\gamma_t\bigg)\theta_{ij},\, \forall i\le j<t\\
\theta_{it} &=& -\gamma_t\sum_{j\ge i}^{t-1}\theta_{ij}\phi_{\omega_j'}(x_t)\phi_{\omega_t'}(x_t),\, \forall i<t\\
\theta_{tt} &=& \gamma_t\phi_{\omega_t}(x_t)\phi_{\omega'_t}(x_t).
\end{eqnarray*}
Please refer to Appendix~\ref{appendix:gp_update_rule} for the details of the derivation. 

\item Assume that the testing points, $\{x_i^{*}\}_{i=1}^m$, are given beforehand, instead of approximating the operator $\Acal$, we target on functions $F^* = [f^*_1, \ldots, f^*_m]^\top$ where $f^*_i(\cdot) = k(\cdot)^\top  \rbr{K + \sigma^2I}^{-1}k_i^*$, $k(\cdot) = [k(x_1, \cdot), \ldots, k(x_2, \cdot)]$ and $k_i^* = [k(x_i^*, x_1),\ldots, k(x_i^*, x_n)]^\top$. Estimating $f^*_i(\cdot)$ can be accomplished by solving the optimization problem~(\ref{eq:primal}) with square loss and setting $y_j = k(x_i^*, x_j),\forall j= 1,\ldots,n$, $\nu = 2\sigma^2$, leading to the same update rule as kernel ridge regression. 
\end{enumerate}

After we obtain these estimators, we can calculate the predictive variance on $x_i^*$ by either $k(x_i^*, x_i^*) - \Acal(x_i^*, x_i^*)$ or $k(x_i^*, x_i^*) - f_i^*(x_i^*)$. We conduct experiments to justify the novel formulations for approximating both the mean and variance of posterior of Gaussian processes for regression, and the doubly stochastic update rule in Section.(\ref{sec:experiments}).

Note that, to approximate the operator $\Acal$, doubly stochastic gradient requires $O(t^2)$ memory. Although we do not need to save the whole training dataset, which saves $O(dt)$ memory cost, this is still computationally expensive. When the $m$ testing data are given, we estimate $m$ functions and each of them requires $O(t)$ memory cost, the total cost will be $O(tm)$ by the second algorithm.

\section{Theoretical Guarantees}\label{sec:analysis}

In this section, we will show that, both in expectation and with high probability, our algorithm can estimate the optimal function in the RKHS with rate $O(1/t)$, and achieve a generalization bound of $O(1/\sqrt{t})$. The analysis for our algorithm has a new twist compared to previous analysis of stochastic gradient descent algorithms, since the random feature approximation results in an estimator which is outside the RKHS. Besides the analysis for stochastic functional gradient descent, we need to use martingales and the corresponding concentration inequalities to prove that the sequence of estimators, $f_{t+1}$, outside the RKHS converge to the optimal function, $f_\ast$, in the RKHS. We make the following standard assumptions ahead for later references:
\begin{enumerate}[noitemsep, nolistsep, label=\Alph*.]
\item There exists an optimal solution, denoted as $f_*$, to the problem of our interest (\ref{eq:primal}).
\item Loss function $\ell(u,y):\RR\times \RR\to\RR$ and its first-order derivative is $L$-Lipschitz continous in terms of the first argument.
\item  For any data $\{(x_i,y_i)\}_{i=1}^t$ and any trajectory $\{f_i(\cdot)\}_{i=1}^{t}$, there exists $M>0$, such that $|\ell'(f_i(x_i),y_i)|\leqslant M$.
Note  in our situation $M$ exists and $M<\infty$ since we assume bounded domain and the functions $f_t$ we generate are always bounded as well.
\item There exists $\kappa>0$ and $\phi>0$, such that
$k(x,x')\leqslant \kappa,\, |\phi_\omega(x) \phi_\omega(x')|\leqslant \phi,\forall x,x'\in\Xcal,\omega\in\Omega.$
For example, when $k(\cdot,\cdot)$ is the Gaussian RBF kernel, we have $\kappa=1$, $\phi=2$.\\[-5mm]
\end{enumerate}
We now present our main theorems as below. Due to the space restrictions, we will only provide a short sketch of proofs here. The full proofs for the these theorems are given in the Appendix~\ref{appendix:proof_details}-\ref{appendix:suboptimality}.
\begin{theorem}[Convergence in expectation]\label{thm:expectation}
When $\gamma_t=\frac{\theta}{t}$ with $\theta>0$ such that $\theta\nu\in(1,2)\cup\ZZ_+$,
$$\EE_{\Dcal^t,\omegab^t}\sbr{  |f_{t+1}(x) - f_\ast(x)|^2 }\leqslant \frac{2C^2+2\kappa Q_1^2}{t},\quad \text{for any $x\in \Xcal$}$$
where $Q_1= \max\cbr{\nbr{f_\ast}_\Hcal, (Q_0 + \sqrt{ Q_0^2 + (2\theta\nu -1) (1+ \theta\nu)^2\theta^2  \kappa M^2 } )/(2\nu \theta - 1)}$, with $Q_0 =2\sqrt{2} \kappa^{1/2} (\kappa + \phi)LM \theta^2$, and $C^2=4(\kappa+\phi)^2M^2\theta^2$.
\end{theorem}
\begin{theorem}[Convergence with high probability]\label{thm:probability}
When $\gamma_t=\frac{\theta}{t}$ with $\theta>0$ such that $\theta\nu\in\ZZ_+$ and $t\ge \theta\nu$,  for any $x\in \Xcal$, we have with probability at least $1-3\delta$ over $(\Dcal^t,\omegab^t)$,
$$  |f_{t+1}(x) - f_\ast(x)|^2 \leqslant  \frac{C^2\ln(2/\delta)}{t}+\frac{2\kappa Q_2^2\ln(2t/\delta)\ln^2(t)}{t}, $$
where $C$ is as above and $Q_2= \max\cbr{\nbr{f_\ast}_\Hcal, Q_0 + \sqrt{ Q_0^2 + \kappa M^2(1+\theta\nu)^2(\theta^2+16\theta/\nu) }}$, with $ Q_0 =4\sqrt{2}\kappa^{1/2}M\theta(8+(\kappa+\phi)\theta L)$.
\end{theorem}
\begin{proof}\textbf{sketch:} We focus on the convergence in expectation; the high probability bound can be established in a similar fashion.
The main technical difficulty is that $f_{t+1}$ may not be in the RKHS $\Hcal$. The key of the proof is then to construct an intermediate function $h_{t+1}$, such that the difference between $f_{t+1}$ and $h_{t+1}$ and the difference  between $h_{t+1}$ and $f_*$ can be bounded. More specifically,
\begin{align}
  h_{t+1}(\cdot) = h_{t}(\cdot) - \gamma_t ( \xi_t(\cdot) + \nu h_t(\cdot) ) = \sum\nolimits_{i=1}^t a_t^i \xi_i(\cdot),\quad \forall t > 1, \quad\text{and}\quad h_1(\cdot) = 0,
\end{align}
where $\xi_t(\cdot) = \EE_{\omega_t}[\zeta_t(\cdot)]$. Then for any $x$, the error can be decomposed as two terms
\begin{eqnarray*}
  |f_{t+1}(x) - f_\ast(x)|^2
  \leqslant 2 \underbrace{ |f_{t+1}(x) - h_{t+1}(x)|^2 }_{\text{ error due to random features}}~~~~+~~~~2 \kappa \underbrace{\nbr{ h_{t+1} -  f_\ast}_{\Hcal}^2}_{\text{error due to random data}}
\end{eqnarray*}

For the error term due to random features, $h_{t+1}$ is constructed such that $f_{t+1}-h_{t+1}$ is a martingale, and the stepsizes are chosen such that $|a_t^i|\leqslant \frac{\theta}{t}$, which allows us to bound the martingale. In other words, the choices of the stepsizes keep $f_{t+1}$  close to the RKHS. For the error term due to random data, since $h_{t+1} \in \Hcal$, we can now apply the standard arguments for stochastic approximation in the RKHS. Due to the additional randomness, the recursion is slightly more complicated,
$
e_{t+1}\leqslant \rbr{1-  \frac{2\nu\theta}{t}} e_t  +\frac{\beta_1}{t}\sqrt{\frac{e_t}{t} }+\frac{ \beta_2}{t^2},
$
where $e_{t+1}=\EE_{\Dcal^{t},\omegab^{t}}[\|h_{t+1}-f_*\|_\Hcal^2]$, and $\beta_1$ and $\beta_2$ depends on the related parameters. Solving this recursion then leads to a bound for the second error term.
\end{proof}

\begin{figure}
\begin{center}
  \includegraphics[width=0.3\textwidth]{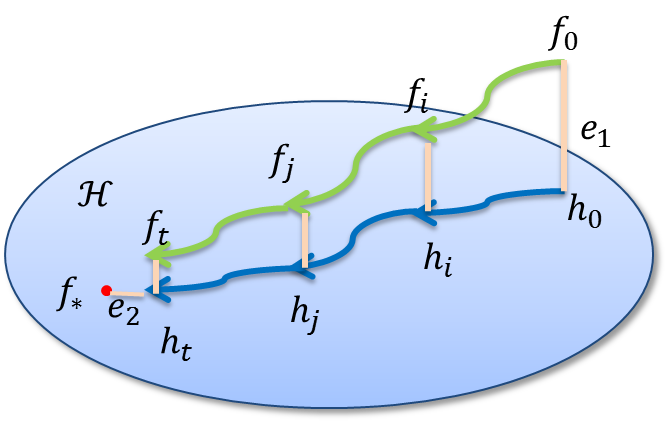}
\end{center}
\caption{$e_1$ stands the error due to random features, and $e_2$ stands for the error due to random data.}
\end{figure}

\begin{theorem}[Generalization bound]\label{thm:risk}
Let the true risk be $R_{true}(f)=\EE_{(x,y)}\sbr{l(f(x),y)}$. Then with probability at least $1-3\delta$ over $(\Dcal^t,\omegab^t)$, and $C$ and $Q_2$ defined as previously
$$R_{true}(f_{t+1})-R_{true}(f_*)\leqslant \frac{(C\sqrt{\ln(8\sqrt{e}t/\delta)}+\sqrt{2\kappa}Q_2\sqrt{\ln(2t/\delta)}\ln(t))L}{\sqrt{t}}.$$\vspace{-5mm}
\end{theorem}
\begin{proof}
By the Lipschitz continuity of $l(\cdot,y)$ and Jensen's Inequality, we have
$$R_{true}(f_{t+1})-R_{true}(f_*)\leqslant L\EE_{x}|f_{t+1}(x)-f_*(x)|\leqslant L\sqrt{\EE_{x}|f_{t+1}(x)-f_*(x)|^2} = L\|f_{t+1}-f_*\|_2.$$
Again, $\|f_{t+1}-f_*\|_2$ can be decomposed as  two terms $O\rbr{\|f_{t+1} - h_{t+1}\|_2^2}$ and $O(\nbr{ h_{t+1} -  f_\ast}_{\Hcal}^2)$, which can be bounded similarly as in Theorem~\ref{thm:probability} (see Corollary~\ref{cor:L2} in the appendix).
\end{proof}
{\bf Remarks.} The overall rate of convergence in expectation, which is $O(1/t)$, is indeed optimal.  Classical complexity theory (see, e.g. reference in~\cite{NemJudLanSha09}) shows that to obtain $\epsilon$-accuracy solution, the number of iterations needed for the stochastic approximation is  $\Omega(1/\epsilon)$ for strongly convex case and $\Omega(1/\epsilon^2)$ for general convex case. Different from the classical setting of stochastic approximation, our case imposes not one but two sources of randomness/stochasticity in the gradient, which intuitively speaking, might require higher order number of iterations for general convex case. However, the variance of the random features only contributes \emph{additively} to the constant in the final convergence rate. Therefore, our method is still able to achieve the same rate as in the classical setting. Notice that these bounds are achieved by adopting the classical stochastic 
gradient algorithm, and they may be further refined with more sophisticated techniques and analysis. For example, techniques for reducing variance of SGD proposed in~\cite{JohZha013}, mini-batch and preconditioning~\cite{AgaKakKarSonVal14, YanJinZhu14} can be used to reduce the constant factors in the bound significantly. Theorem \ref{thm:expectation} also reveals  bounds in $L_{\infty}$ and $L_2$ sense as in Appendix~\ref{appendix:L2}. The choices of stepsizes $\gamma_t$ and the tuning parameters given in these bounds are only for sufficient conditions and simple analysis; other choices can also lead to bounds in the same order.

\section{Computation, Memory and Statistics Trade-off}\label{sec:comp_analysis}

To investigate computation, memory and statistics trade-off, we will fix the desired $L_2$ error in the function estimation to $\epsilon$,~\ie, $\|f-f_*\|_2^2 \leqslant \epsilon$, and work out the dependency of other quantities on $\epsilon$. These other quantities include the preprocessing time, the number of samples and random features (or rank), the number of iterations of each algorithm, and the computational cost and memory requirement for learning and prediction. We assume that the number of samples, $n$, needed to achieve the prescribed error $\epsilon$ is of the order $O(1/\epsilon)$, the same for all methods. Furthermore, we make no other regularity assumption about margin properties or the kernel matrix such as fast spectrum decay. Thus the required number of random feature (or ranks), $r$, will be of the order $O(n)=O(1/\epsilon)$~\cite{DriMah05,CorMohTal10,RahRec09,LopSraSmoGhaetal14}.

We will pick a few representative algorithms for comparison, namely, \emph{(i)} NORMA~\cite{KivSmoWil04b}: kernel methods trained with stochastic functional gradients; \emph{(ii)} k-SDCA~\cite{ShaTon13b}: kernelized version of stochastic dual coordinate ascend; \emph{(iii)} r-SDCA: first approximate the kernel function with random features, and then run stochastic dual coordinate ascend; \emph{(iv)} n-SDCA: first approximate the kernel matrix using Nystr{\"o}m's method, and then run stochastic dual coordinate ascend; similarly we will combine Pegasos algorithm~\cite{ShaSinSre07}, stochastic block mirror descent (SBMD)~\cite{DanLan13}, and random block coordinate descent (RBCD)~\cite{Nesterov12} with random features and Nystr{\"o}m's method, and obtain \emph{(v)} r-Pegasos, \emph{(vi)} n-Pegasos, \emph{(vii)} r-SBMD, \emph{(viii)} n-SBMD, \emph{(ix)} r-RBCD, and \emph{(x)} n-RBCD, respectively. The comparisons are summarized below in Table.~\ref{table:tradeoff}\footnote{We only considered general kernel 
algorithms in this section. For some specific loss functions, \eg, hinge-loss, there are algorithms proposed to achieve better memory saving with extra training cost, such as support vector reduction technique~\cite{CotShaSre13}.}

\begin{table}[tb!]
\vspace{-5mm}
\caption{Comparison of Computation and Memory Requirements}
\label{table:tradeoff}
\begin{center}
\begin{tabular}{c|c|cc|cc}
\hline
Algorithms  &\multicolumn{1}{c|}{Preprocessing} &\multicolumn{2}{c|}{Total Computation Cost} &\multicolumn{2}{c}{Total Memory Cost}\\
\cline{3-6}
                                  &Computation   &Training              &Prediction     &Training &Prediction\\
\hline
Doubly SGD    &$O(1)$  &$O(d/\epsilon^2)$ &$O(d/\epsilon)$ &$O(1/\epsilon)$ &$O(1/\epsilon)$\\
NORMA    &$O(1)$  &$O(d/\epsilon^2)$ &$O(d/\epsilon)$ &$O(d/\epsilon)$ &$O(d/\epsilon)$\\
k-SDCA   &$O(1)$  &$O(d/\epsilon^2\log(\frac{1}{\epsilon}))$ &$O(d/\epsilon)$ &$O(d/\epsilon)$ &$O(d/\epsilon)$\\
r-SDCA   &$O(1)$  &$O(d/\epsilon^2\log(\frac{1}{\epsilon}))$ &$O(d/\epsilon)$ &$O(1/\epsilon)$ &$O(1/\epsilon)$ \\
n-SDCA   &$O(1/\epsilon^3)$  &$O(d/\epsilon^2\log(\frac{1}{\epsilon}))$ &$O(d/\epsilon)$ &$O(1/\epsilon)$       &$O(1/\epsilon)$ \\
r-Pegasos   &$O(1)$  &$O(d/\epsilon^2)$ &$O(d/\epsilon)$ &$O(1/\epsilon)$ &$O(1/\epsilon)$ \\
n-Pegasos   &$O(1/\epsilon^3)$  &$O(d/\epsilon^2)$ &$O(d/\epsilon)$ &$O(1/\epsilon)$       &$O(1/\epsilon)$ \\
r-SBMD  &$O(1)$  &$O(d/\epsilon^2)$ &$O(d/\epsilon)$ &$O(1/\epsilon)$       &$O(1/\epsilon)$ \\
n-SBMD  &$O(1/\epsilon^3)$  &$O(d/\epsilon^2)$ &$O(d/\epsilon)$ &$O(1/\epsilon)$       &$O(1/\epsilon)$ \\
r-RBCD &$O(1)$  &$O(d/\epsilon^2\log(\frac{1}{\epsilon}))$ &$O(d/\epsilon)$ &$O(1/\epsilon)$       &$O(1/\epsilon)$ \\
n-RBCD &$O(1/\epsilon^3)$  &$O(d/\epsilon^2\log(\frac{1}{\epsilon}))$ &$O(d/\epsilon)$ &$O(1/\epsilon)$       &$O(1/\epsilon)$ \\
\hline
\end{tabular}
\vspace{-5mm}
\end{center}
\end{table}
\begin{table}[tb!]
\caption{Comparison of Computation and Memory Requirement Per Iteration. $b$ denotes the block size in algorithms SBMD and RBCD.}
\vspace{-5mm}
\label{table:res_per_iter}
\begin{center}
\begin{tabular}{c|c|c|c}
\hline
Algorithms  &Computation per Iteration     & Memory per Iteration &Iteration \#\\
\hline
Doubly SGD &$\Theta(dt + t + t)$  &$\Theta(t)$ & $O(1/\epsilon)$\\
r-SDCA     &$\Theta(dn + n + n)$  &$\Theta(n)$ & $O(1/\epsilon\log(\frac{1}{\epsilon}))$\\
r-Pegasos  &$\Theta(dn + n + n)$  &$\Theta(n)$ & $O(1/\epsilon)$\\
r-SBMD     &$\Theta(dn + n + n/b)$ &$\Theta(n)$ & $O(b/\epsilon)$\\
r-RBCD     &$\Theta(dn^2+n^2+n/b)$ &$\Theta(n)$ & $O(\log(1/\epsilon))$\\
\hline
\end{tabular}
\vspace{-5mm}
\end{center}
\end{table}
From Table~\ref{table:tradeoff}, one can see that our method, r-SDCA, r-Pegasos, r-SBMD and r-RBCD achieve the best dependency on the dimension, $d$, of the data up to a log factor. However, often one is interested in increasing the number of random features as more data points are observed to obtain a better generalization ability, \eg, {\bf in streaming setting}. Then special procedures need to be designed for updating the r-SDCA, r-Pegasos, r-SBMD and r-RBCD solutions, which is not clear how to do easily and efficiently with theoretical guarantees. As a more refined comparison, our algorithm is also the cheapest in terms of per training iteration computation and memory requirement. We list the computational and memory requirements at a particular iteration $t < n$ for these five algorithms to achieve $\epsilon$ error in Table~\ref{table:res_per_iter}.

\section{Experiments}\label{sec:experiments}
We show that our method compares favorably to other scalable kernel methods in medium scale datasets, and neural nets in large scale datasets. Below is a summary of the datasets used. A ``yes'' for the last column means that virtual examples (random cropping and mirror imaging of the original pictures) are generated for training. K-ridge stands for kernel ridge regression; GPR stands for Gaussian processes regression; K-SVM stands for kernel SVM; K-logistic stands for kernel logistic regression.

\begin{table}[t!]
\begin{center}
\caption{Datasets}\label{table:datasets}
\renewcommand{\tabcolsep}{4pt}
\begin{tabular}{cc|c|cccc}
  \hline
  & Name & Model & \# of samples & Input dim & Output range & Virtual\\
  \hline
  (1) & Synthetic& GPR & $2^{11}$ & $2$ & $[-1, 1.3]$ & no\\
  (2) & Synthetic& K-ridge& $2^{20}$ & $2$ & $[-1, 1.3]$ & no\\
  (3) & Adult & K-SVM & 32K & 123 &$\cbr{-1, 1}$ & no\\
  (4) & MNIST 8M 8 vs. 6 & K-SVM&1.6M & 784 & $\cbr{-1, 1}$ & yes\\
  (5) & Forest  & K-SVM & 0.5M & 54 & $\cbr{-1, 1}$ & no\\
  (6) & MNIST 8M & K-logistic& 8M &1568 & $\cbr{0,\ldots,9}$ & yes\\
  (7) & CIFAR 10 & K-logistic& 60K &2304 & $\cbr{0,\ldots,9}$ & yes\\
  (8) & ImageNet & K-logistic& 1.3M &9216 & $\cbr{0,\ldots,999}$ & yes\\
  (9) & QuantumMachine & K-ridge &6K &276  & $[-800, -2000]$ & yes\\
  (10) & MolecularSpace & K-ridge & 2.3M & 2850 & $[0,13]$ & no\\
  \hline
\end{tabular}
\vspace{-5mm}
\end{center}
\end{table}

\paragraph{\bf Experiment settings.} We first justify the doubly stochastic algorithm for Gaussian processes regression on dataset (1), comparing with NORMA. The dataset is medium size, so that the closed-form for posterior is tractable. For the large-scale datasets (2) --- (5), we compare with the first seven algorithms for solving kernel methods discussed in Table~\ref{table:tradeoff}. For the algorithms based on low rank kernel matrix approximation and random features, \ie, pegasos and SDCA, we set the rank $r$ or number of random features $r$ to be $2^{8}$. We use the same batch size for both our algorithms and the competitors. We adopted two stopping criteria for different purposes. We first stopped the algorithms when they pass through the entire dataset once (SC1). This stopping criterion is designed for justifying our motivation. By investigating the performances of these algorithms with different levels of random feature approximations but the same number of training samples, we could identify that the bottleneck of the performances of the vanilla methods with explicit feature will be their approximation ability. To further demonstrate the advantages of the proposed algorithm in computational cost, we also conduct experiments on datasets (3) -- (5) running the competitors within the same time budget as the proposed algorithm (SC2). We do not count the preprocessing time of Nystr{\"o}m's method for n-Pegasos and n-SDCA, though it takes substantial amount of time. The algorithms are executed on the machine with AMD 16 2.4GHz Opteron CPUs and 200G memory. It should be noticed that this gives advantage to NORMA and k-SDCA which could save all the data in the memory. For fairness, we also record as many random features as the memory allowed.
\begin{figure}[t!]
\vspace{-5mm}
\centering
    \includegraphics[width=0.65\linewidth]{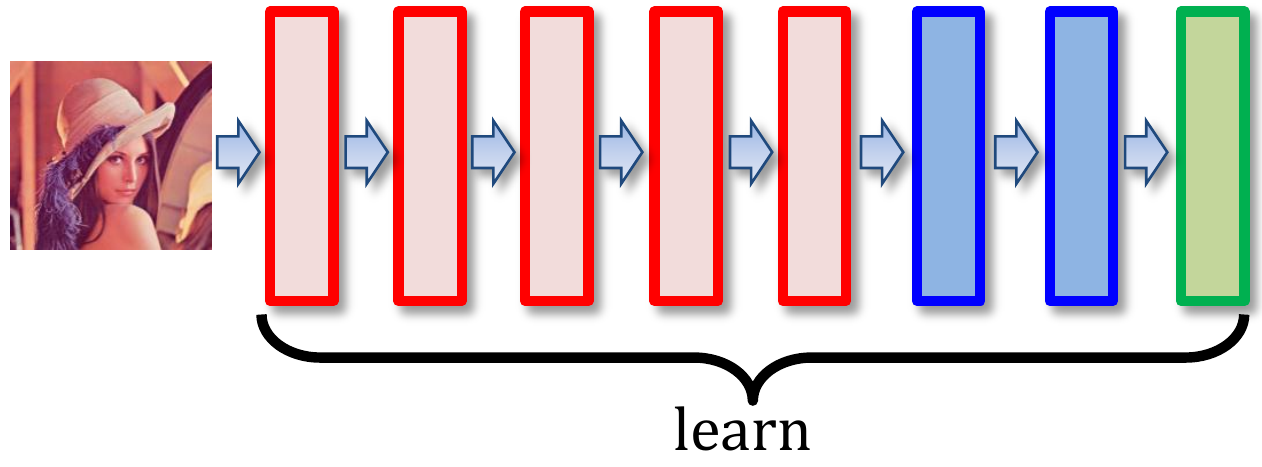} \\
    (1) Jointly Trained Neural Nets\\[1mm]
    \includegraphics[width=0.65\linewidth]{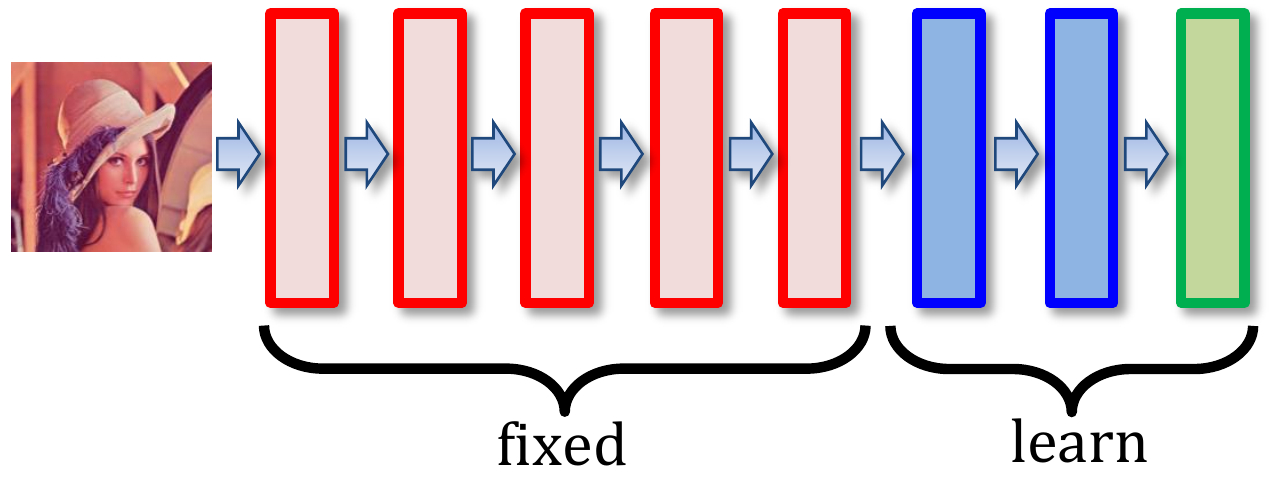} \\
    (2) {Fixed Neural Nets}\\[1mm]
    \includegraphics[width=0.55\linewidth]{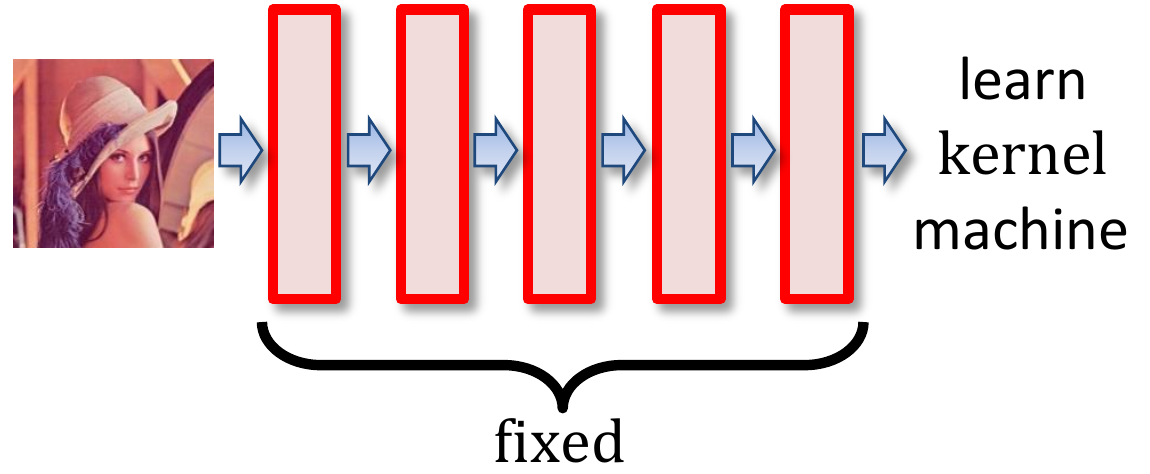} \\
    (3) {Doubly Stochastic Kernel with Fixed Nets}
\caption{Illustration of the neural nets structure in our experiments. The first several {\color{red} red} layers are convolutions with max pooling layers. The following {\color{blue} blue} layers are fully connected layes. The {\color{green} green} layer is the output layer which is multiclass logistic regression model.}
\label{fig:nn_comparison}

\end{figure}
For datasets (6) --- (8), we compare with neural nets for images (``jointly-trained''). In order to directly compare the performance of nonlinear classifiers rather than feature learning abilities, we also use the convolution layers of a trained neural net to extract features, then apply our algorithm and a nonlinear neural net on top to learn classifiers (``fixed''). The structures of these neural nets in Figure~\ref{fig:nn_comparison}. For datasets (9) and (10), we compare with the neural net described in~\cite{MonHanFazRupetal12} and use exactly the same input. In all the experiments, we select the batch size so that for each update, the computation resources can be utilized efficiently.

\subsection{Kernel Ridge Regression}

In this section, we compare our approach with alternative algorithms for kernel ridge regression on 2D synthetic dataset. The data are generated by
\begin{eqnarray*}
y = \cos(0.5\pi \|x\|_2) \exp(-0.1\pi\|x\|_2) + 0.1e
\end{eqnarray*}
where $x\in[-5, 5]^2$ and $e\sim \Ncal(0, 1)$.
We use Gaussian RBF kernel with kernel bandwidth $\sigma$ chosen to be $0.1$ times the median of pairwise distances between data points (median trick). The regularization parameter $\nu$ is set to be $10^{-6}$.  The batch size and feature block are set to be $2^{10}$.

The results are shown in Figure~\ref{fig:syn_dataset_KRR}. In Figure~\ref{fig:syn_dataset_KRR}(1), we plot the optimal functions generating the data. We justify our proof of the convergence rate in Figure~\ref{fig:syn_dataset_KRR}(2). The blue dotted line is a convergence rate of $1/t$ as a guide. $\hat f_t$ denotes the average solution after $t$-iteration, \ie, $\hat f_t(x) = \frac{1}{t}\sum_{i=1}^{t}f_i(x)$. It could be seen that our algorithm indeed converges in the rate of $O({1}/{t})$. In Figure~\ref{fig:syn_dataset_KRR} (3), we compare the first seven algorithms listed in the Table~\ref{table:tradeoff} for solving the kernel ridge regression.

The comparison on synthetic dataset demonstrates the advantages of our algorithm clearly. Our algorithm achieves comparable performance with NORMA, which uses full kernel, in similar time but less memory cost. The pegasos and SDCA using $2^8$ random or Nystr{\"o}m features perform worse.

\begin{figure*}[!t]\vspace{-5mm}
  \begin{tabular}{ccc}
    \includegraphics[width=0.315\columnwidth]{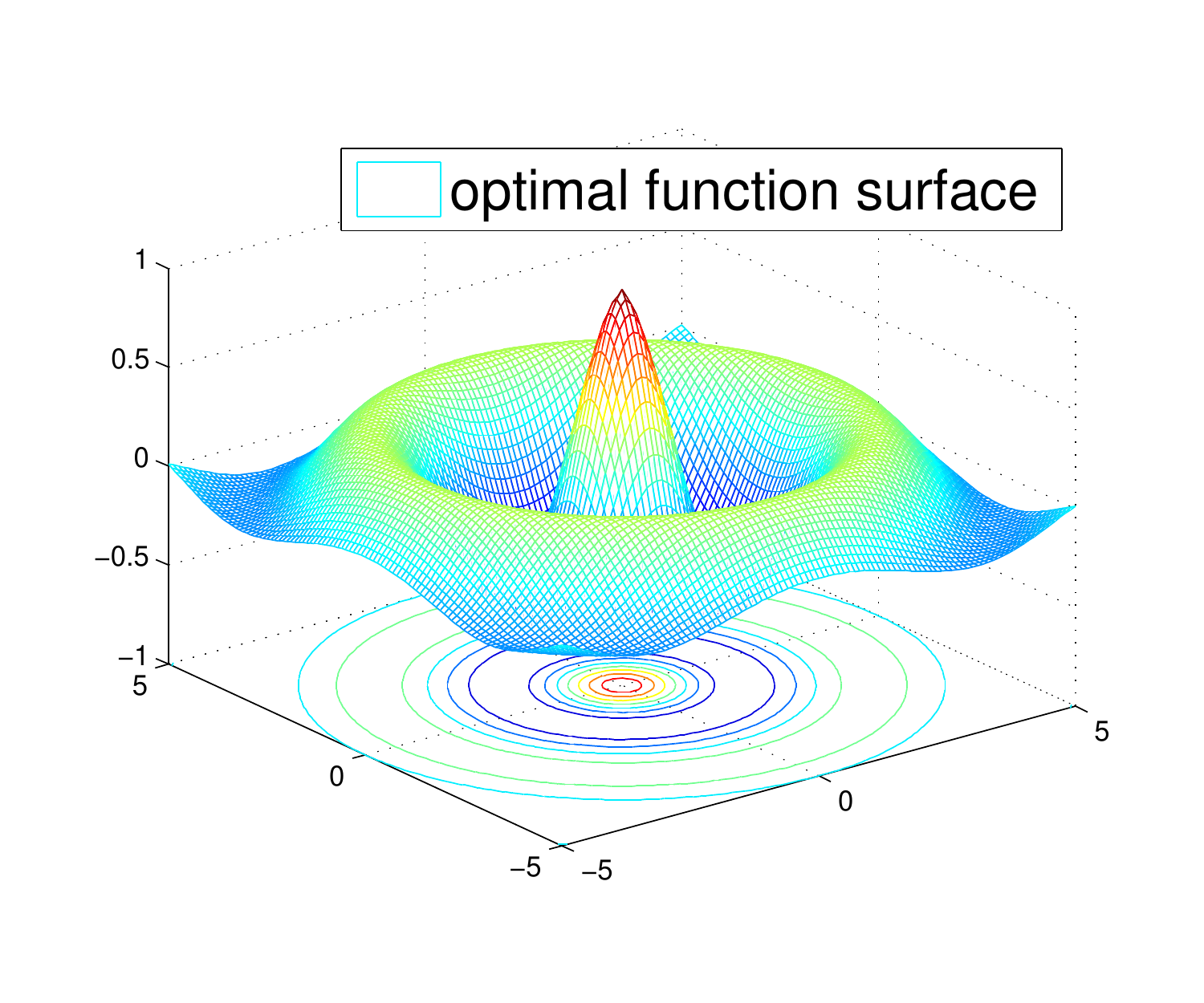} &
    \includegraphics[width=0.315\columnwidth]{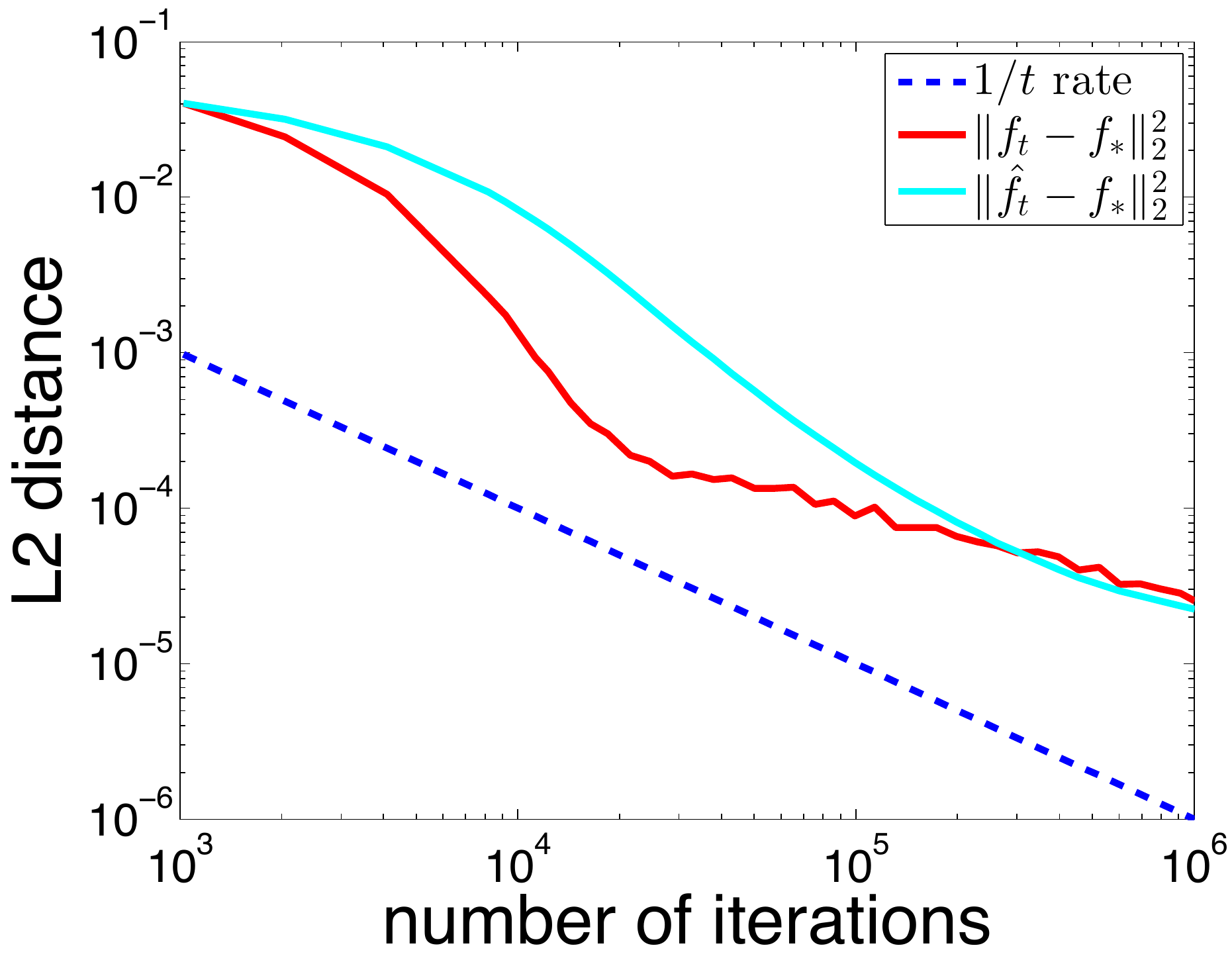} &
    \includegraphics[width=0.315\columnwidth]{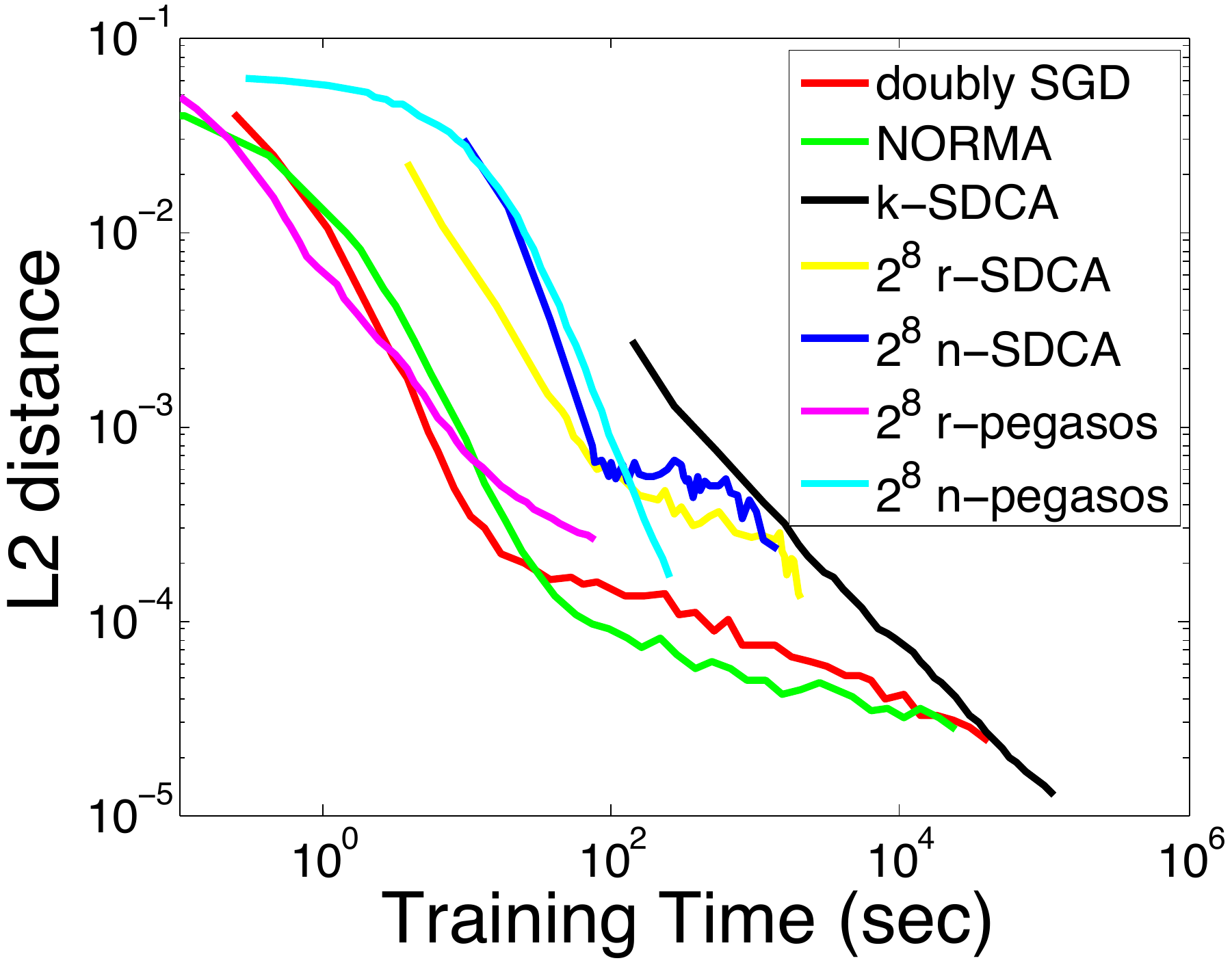} \\
    (1) 2D Synthetic Dataset & (2) Convergence Rate & (3) Accuracy vs. Time\\
  \end{tabular}
  \caption{Experimental results for kernel ridge regression on synthetic dataset.}
  \label{fig:syn_dataset_KRR}
\end{figure*}
\begin{figure*}[!t]\vspace{-5mm}
\centering
  \begin{tabular}{cc}
    \includegraphics[width=0.38\columnwidth]{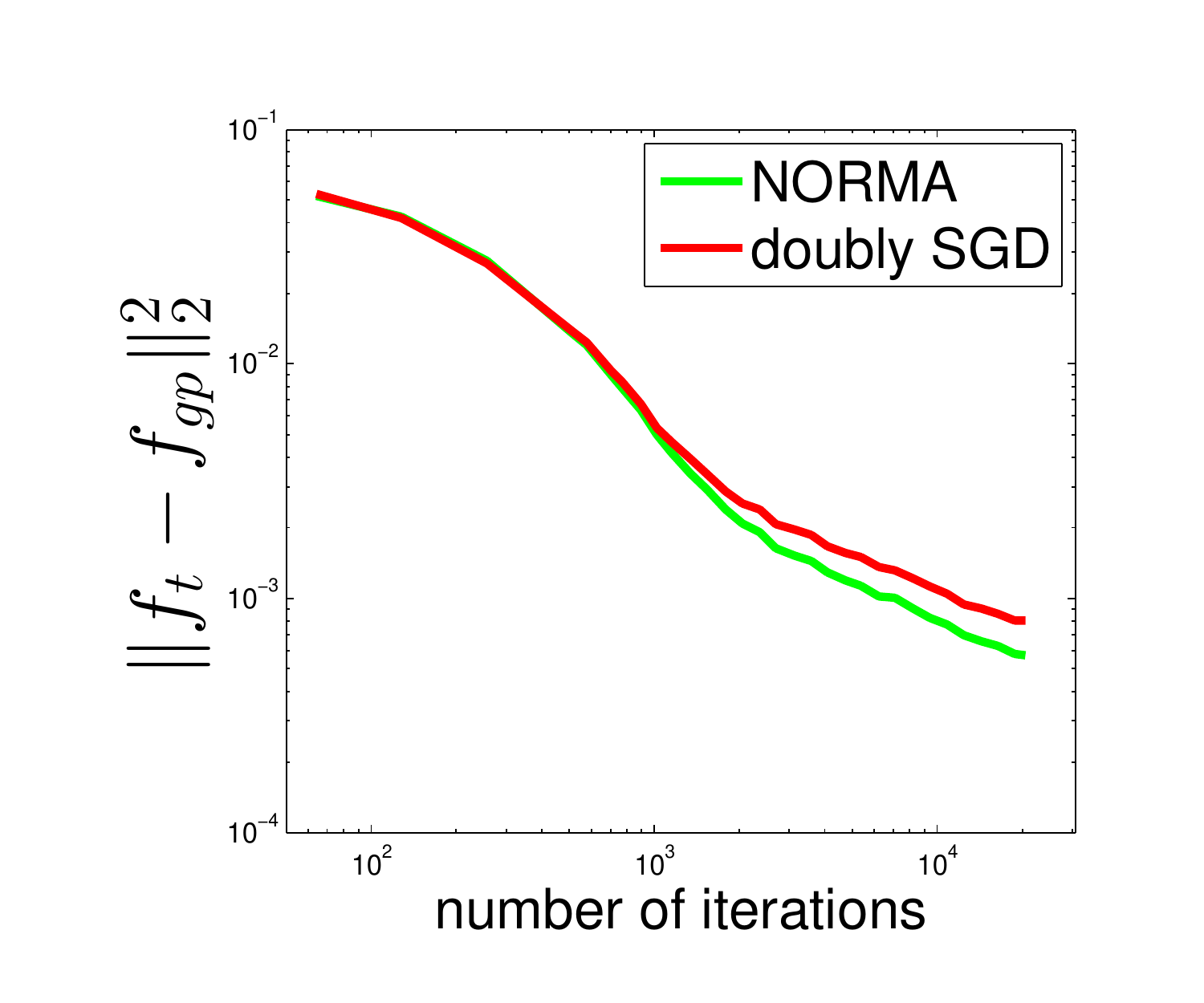} &
    \includegraphics[width=0.38\columnwidth]{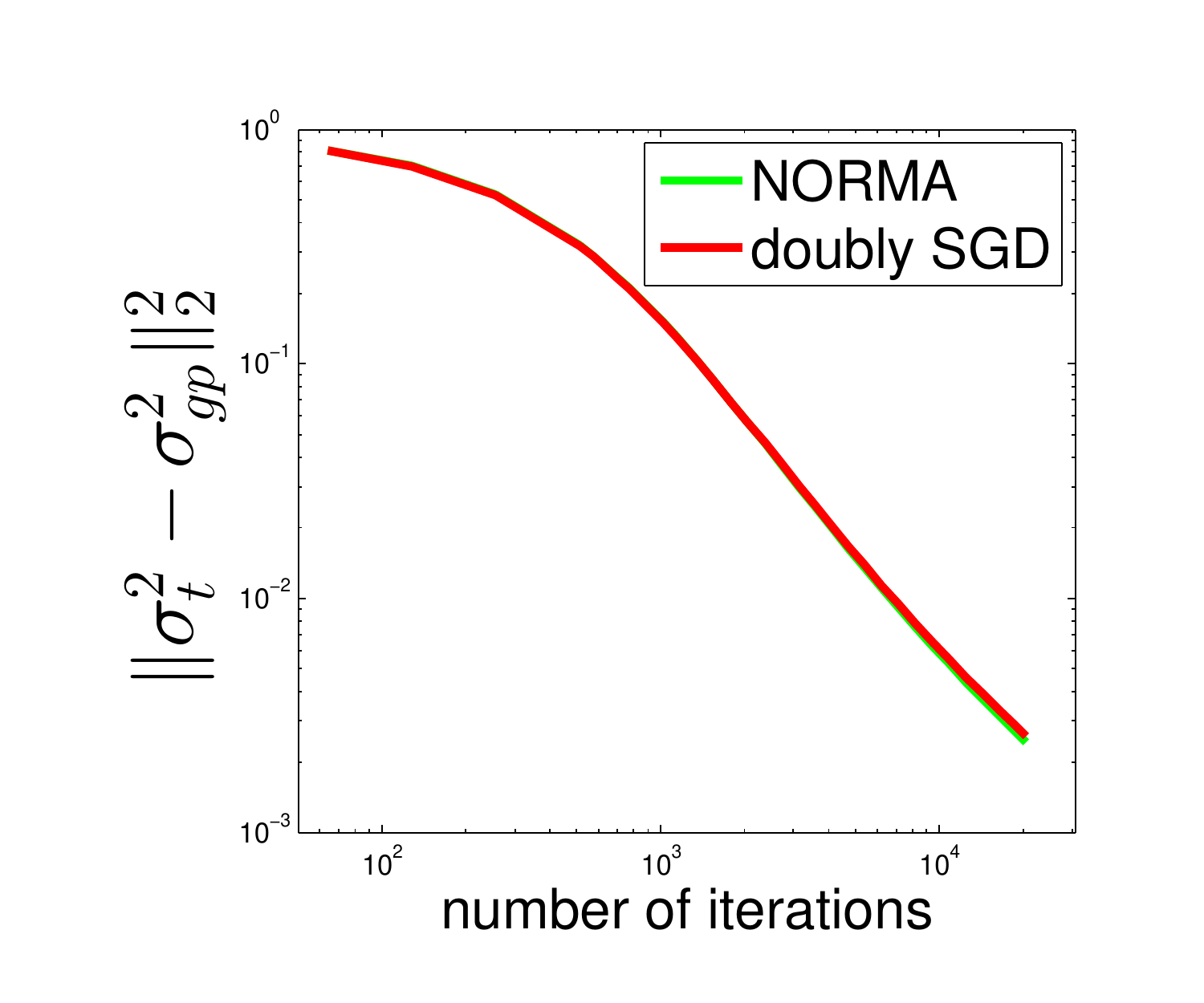} \\
    (1) Posterior Mean Convergence & (2) Posterior Variance Convergence\\
  \end{tabular}
  \caption{Experimental results for Gaussian Processes regression.}
  \label{fig:syn_dataset_GP}
\end{figure*}

\subsection{Gaussian Processes Regression}

As we introduced in Section.~(\ref{sec:doubly_sgd}), the mean and variance of posterior of Gaussian processes for regression problem can be formulated as solutions to some convex optimization problems. We conduct experiments on synthetic dataset for justification. Since the task is computing the posterior, we evaluate the performances by comparing the solutions to the posterior mean and variance, denoted as $f_{gp}$ and $\sigma_{gp}^2$, obtained by closed-form~(\ref{eq:gpr_posterior}). We select $2^{11}$ data from the same model in previous section for training and $2^{10}$ data for testing, so that the closed-form of posterior is tractable. We use Gaussian RBF kernel with kernel bandwidth $\sigma$ chosen by median trick. The noise level $\sigma^2$ is set to be $0.1$. The batch size is set to be $64$ and feature block is set to be $512$. 

We compared the doubly stochastic algorithm with NORMA. The results are shown in Figure~\ref{fig:syn_dataset_GP}. Both the doubly stochastic algorithm and NORMA converge to the posterior, and our algorithm achieves comparable performance with NORMA in approximating both the mean and variance.

\subsection{Kernel Support Vector Machine}

We evaluate our algorithm solving kernel SVM on three datasets (3)--(5) comparing with other several algorithms listed in Table~\ref{table:tradeoff} using stopping criteria SC1 and SC2.

\begin{figure*}[!t]
{\centering
\begin{tabular}{ccc}
    \includegraphics[width=0.31\textwidth]{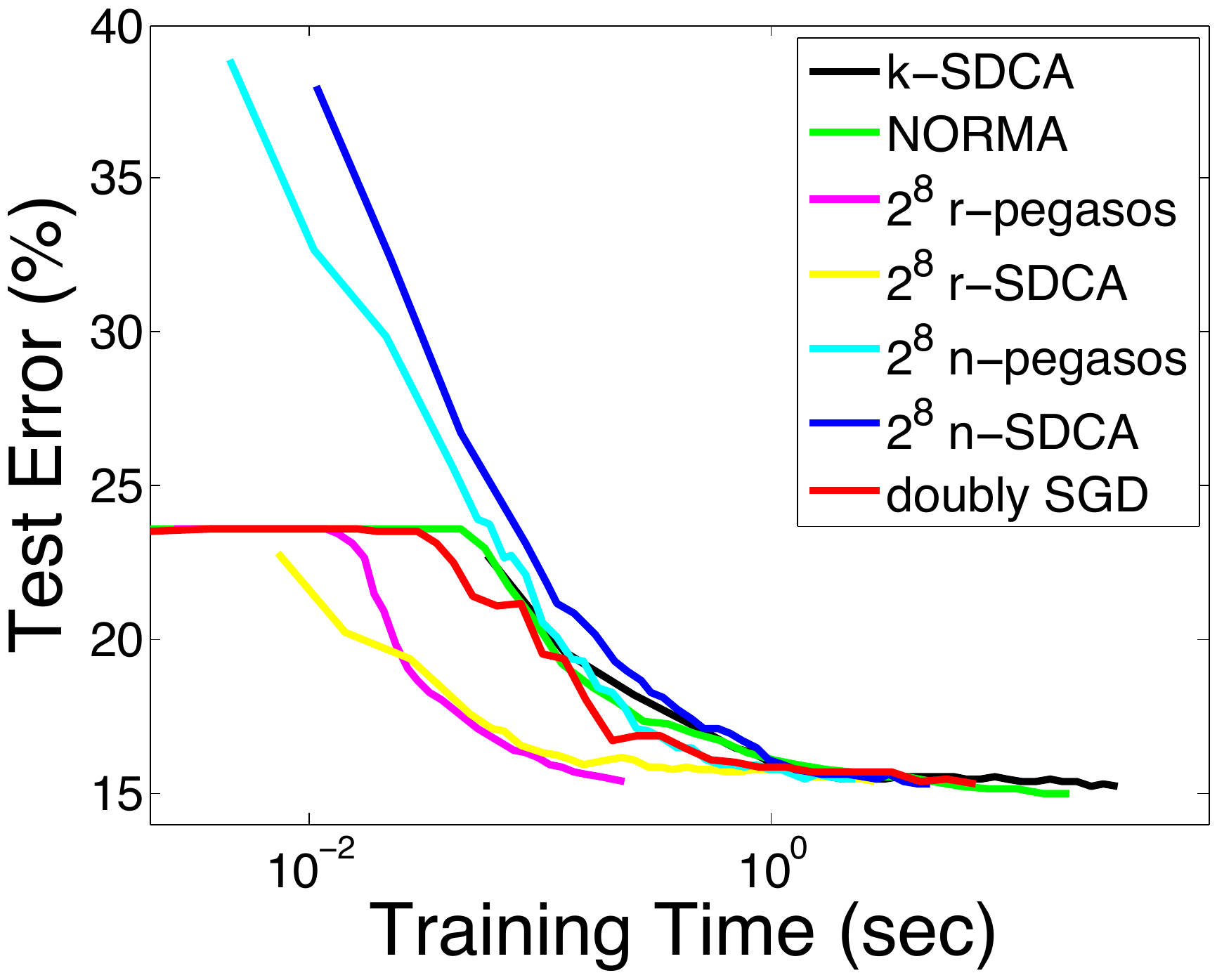} &
    \includegraphics[width=0.316\textwidth]{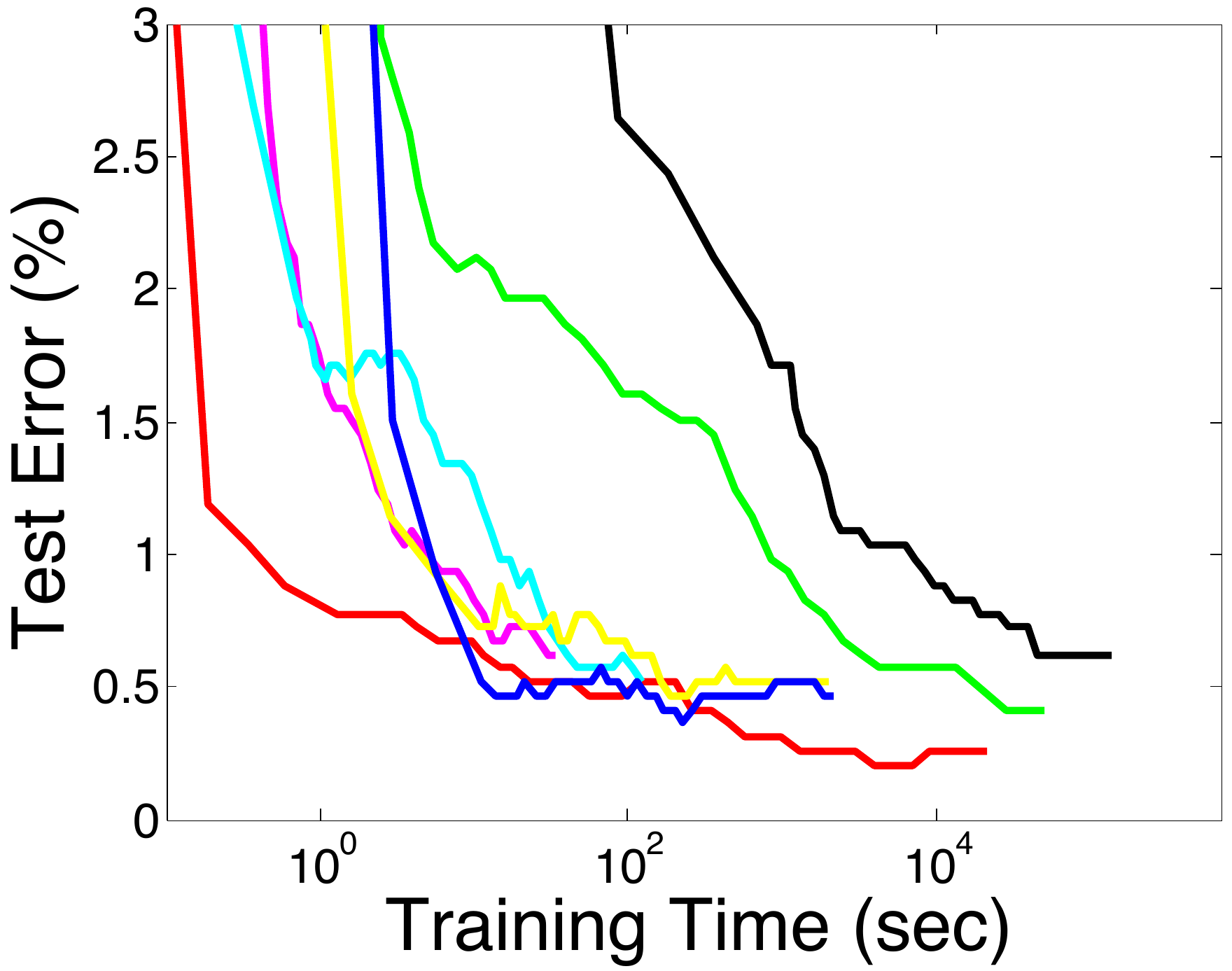}&
    \includegraphics[width=0.313\textwidth]{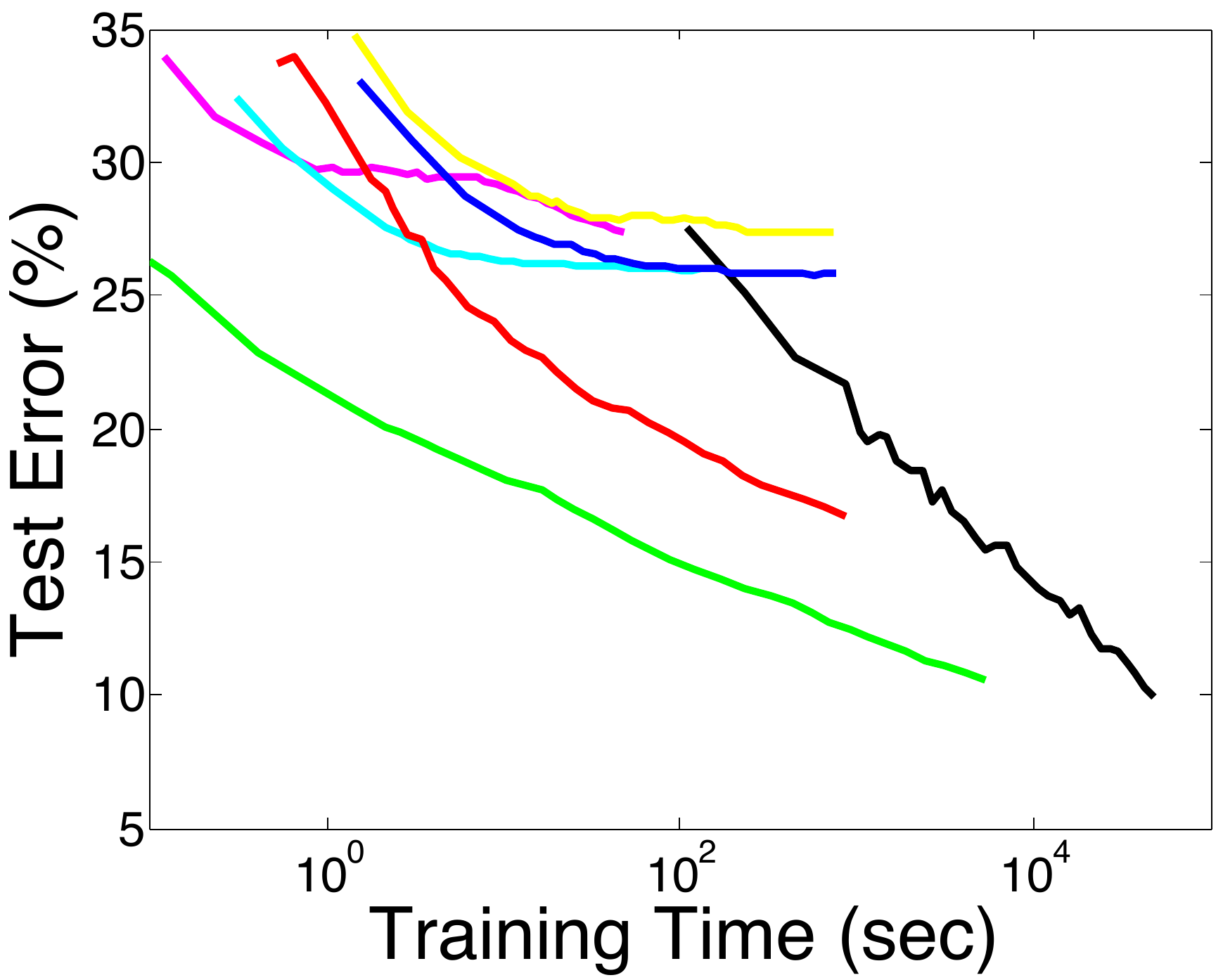}  \\
    SC1:  (1) Adult  & (2) MNIST 8M 8 vs. 6 &(3) Forest \\[-1mm]
  \end{tabular}
  \begin{tabular}{ccc}
    \includegraphics[width=0.315\textwidth]{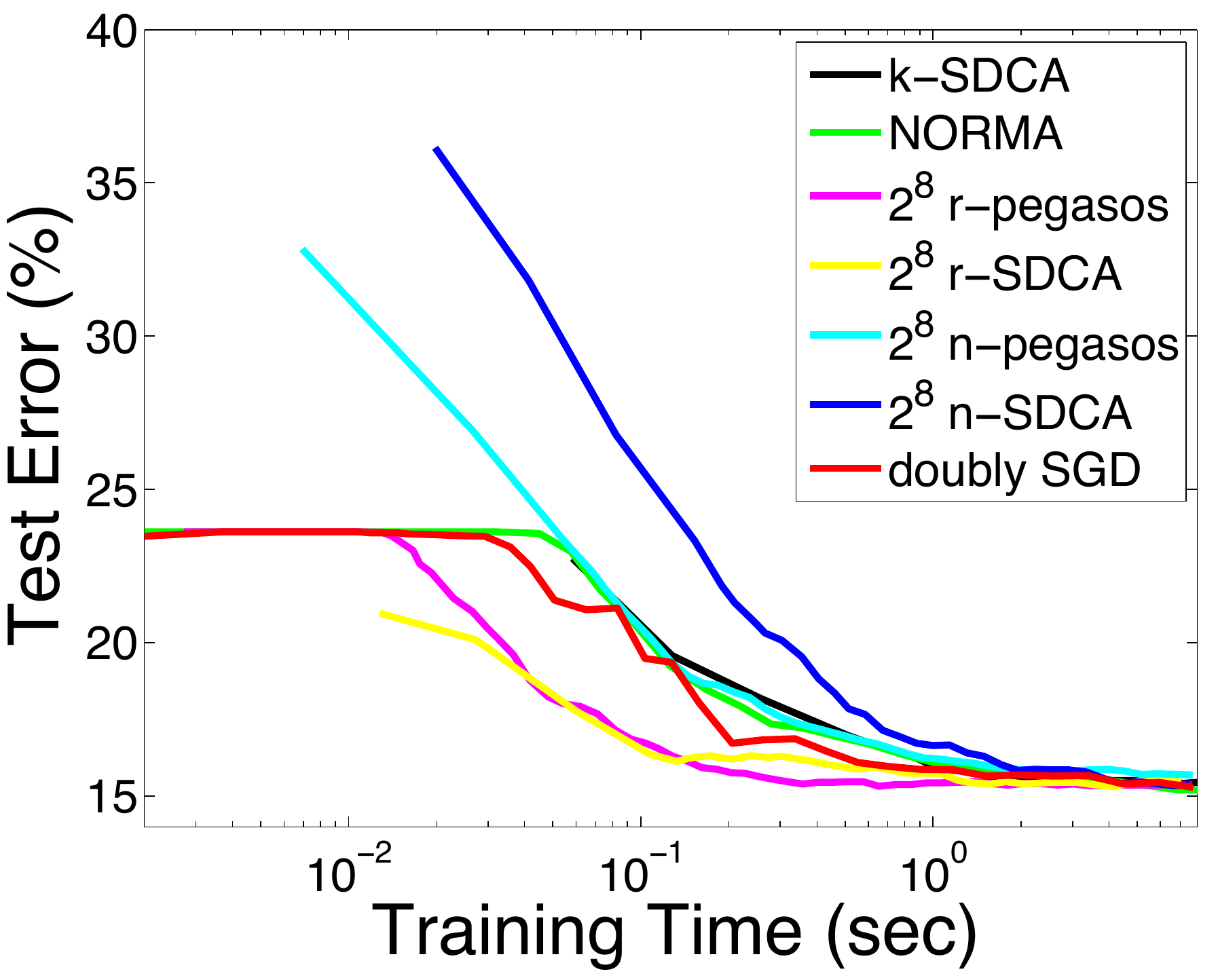} &
    \includegraphics[width=0.315\textwidth]{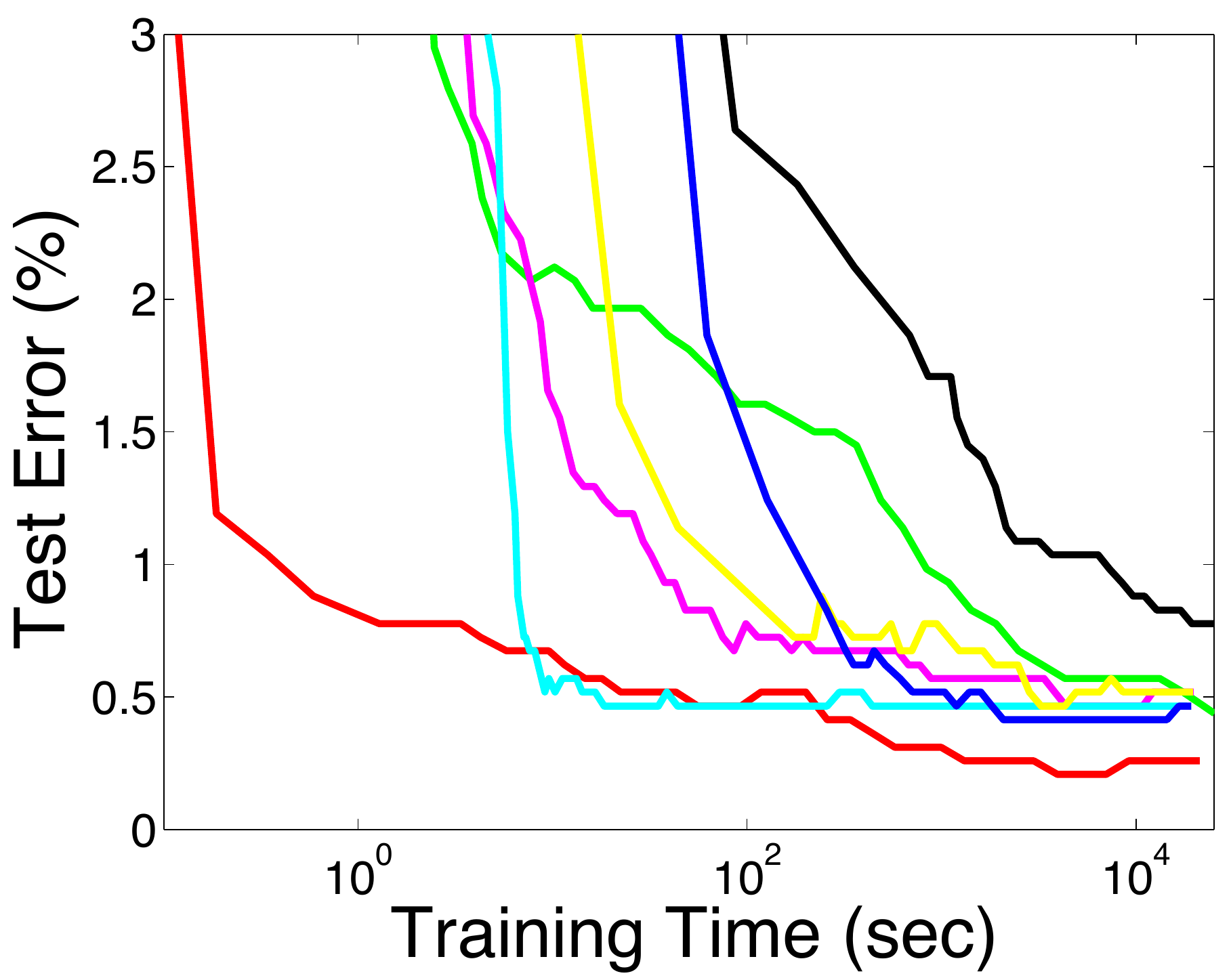} &
    \includegraphics[width=0.315\textwidth]{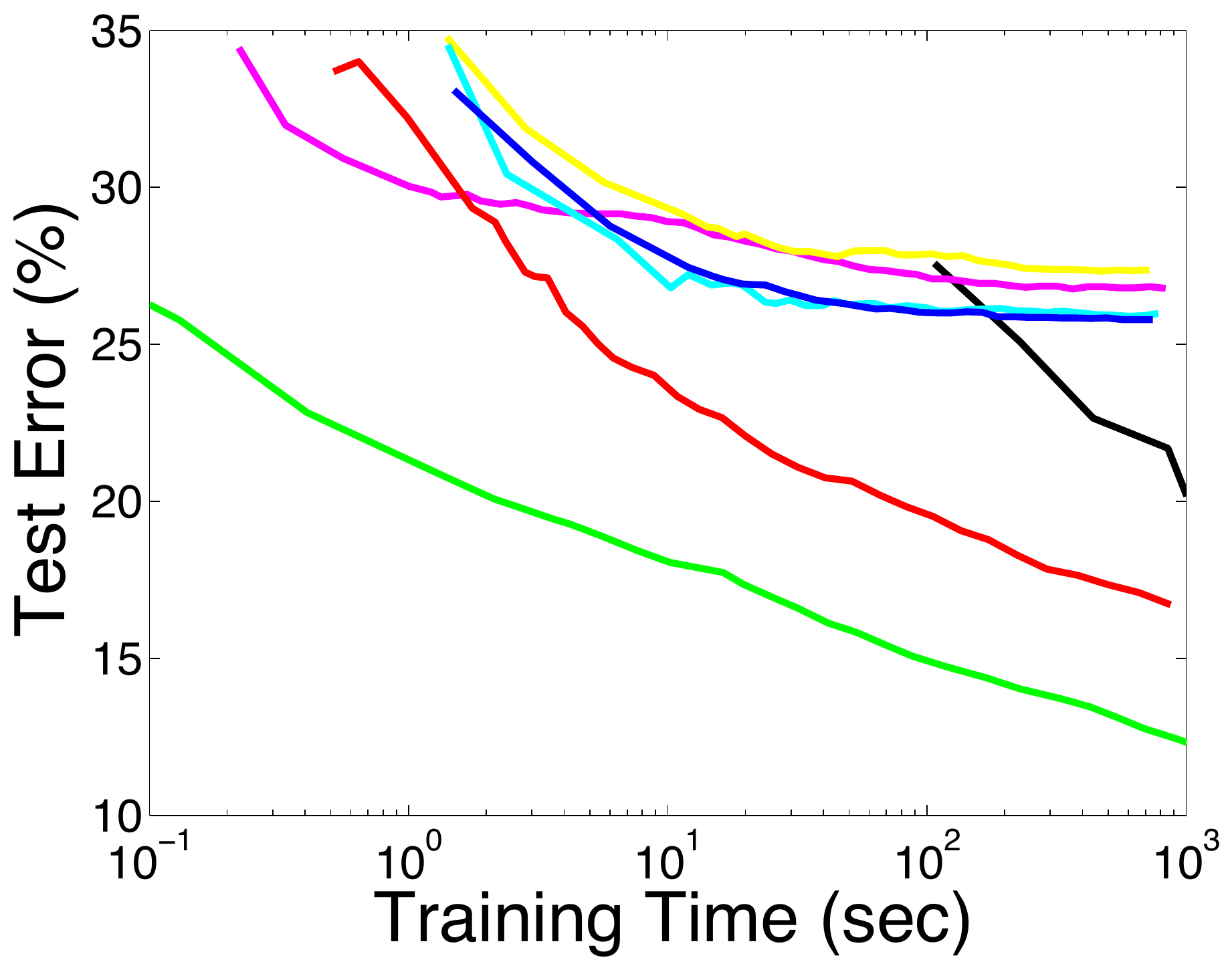} \\
    SC2:   (4) Adult & (5)MNIST 8M 8 vs. 6 & (6) Forest.\\[-3mm]
  \end{tabular}
  \caption{Comparison with other kernel SVM solvers on datasets (3) -- (5) with two different stopping criteria.}
  \label{fig:ksvm_results}}
\end{figure*}

\paragraph{\bf Adult.} We use Gaussian RBF kernel with kernel bandwidth obtained by median trick. The regularization parameter $\nu$ is set to be $1/(100n)$ where $n$ is the number of training samples. We set the batch size to be $2^{6}$ and feature block to be $2^5$. After going through the whole dataset one pass, the best error rate is achieved by NORMA and k-SDCA which is $15\%$ while our algorithm achieves comparable result $15.3\%$. The performances are illustrated in Figure~\ref{fig:ksvm_results}(1). Under the same time budget, all the algorithms perform similarly in Figure~\ref{fig:ksvm_results}(4). The reason of flat region of r-pegasos, NORMA and the proposed method on this dataset is that Adult dataset is unbalanced. There are about $24\%$ positive samples while $76\%$ negative samples. 
\vspace{-3mm}
\paragraph{\bf MNIST 8M 8 vs. 6.}  We first reduce the dimension to 50 by PCA and use Gaussian RBF kernel with kernel bandwidth $\sigma=9.03$ obtained by median trick. The regularization parameter $\nu$ is set to be $1/n$ where $n$ is the number of training samples. We set the batch size to be $2^{10}$ and feature block to be $2^8$. The results are shown in Figure~\ref{fig:ksvm_results}(2) and (5) under SC1 and SC2 respectively. Under both these two stopping criteria, our algorithm achieves the best test error $0.26\%$ using similar training time.
\vspace{-3mm}
\paragraph{\bf Forest.} We use Gaussian RBF kernel with kernel bandwidth obtained by median trick. The regularization parameter $\nu$ is set to be $1/n$ where $n$ is the number of training samples. We set the batch size to be $2^{10}$ and feature block to be $2^8$. In Figure~\ref{fig:ksvm_results}(3), we shows the performances of all algorithms using SC1. NORMA and k-SDCA achieve the best error rate, which is $10\%$, while our algorithm achieves around $15\%$, but still much better than the pegasos and SDCA with $2^8$ features. In the same time budget, the proposed algorithm performs better than all the alternatives except NORMA in Figure~\ref{fig:ksvm_results}(6).

As seen from the performance of pegasos and SDCA on Adult and MNIST, using fewer features does not deteriorate the classification error. This might be because there are cluster structures in these two binary classification datasets. Thus, they prefer low rank approximation rather than full kernel. Different from these two datasets, in the forest dataset, algorithms with full kernel, \ie, NORMA and k-SDCA, achieve best performance. With more random features, our algorithm performs much better than pegasos and SDCA under both SC1 and SC2. Our algorithm is preferable for this scenario, \ie, huge dataset with sophisticated decision boundary. Although utilizing full kernel could achieve better performance, the computation and memory requirement for the kernel on huge dataset are costly. To learn the sophisticated boundary while still considering the computational and memory cost, we need to efficiently approximate the kernel in $O(\frac{1}{\epsilon})$ with $O(n)$ random features at least. Our algorithm could handle so many random features efficiently in both computation and memory cost, while for pegasos and SDCA such operation is prohibitive.

\subsection{Classification Comparisons to Convolution Neural Networks}

\begin{figure*}[!t]
{\centering
   \begin{tabular}{ccc}
       \includegraphics[width=0.315\columnwidth]{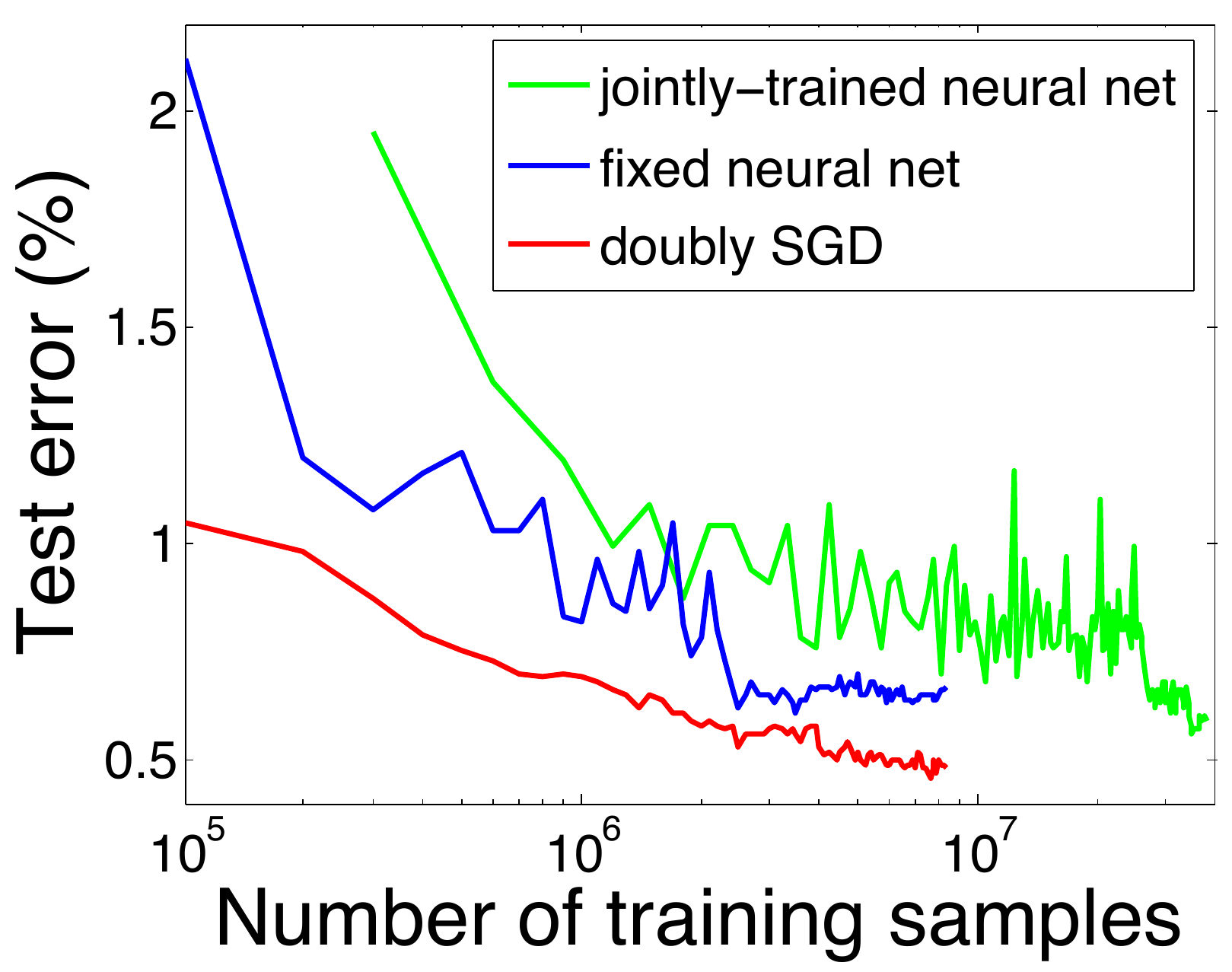}  &
          \includegraphics[width=0.315\columnwidth]{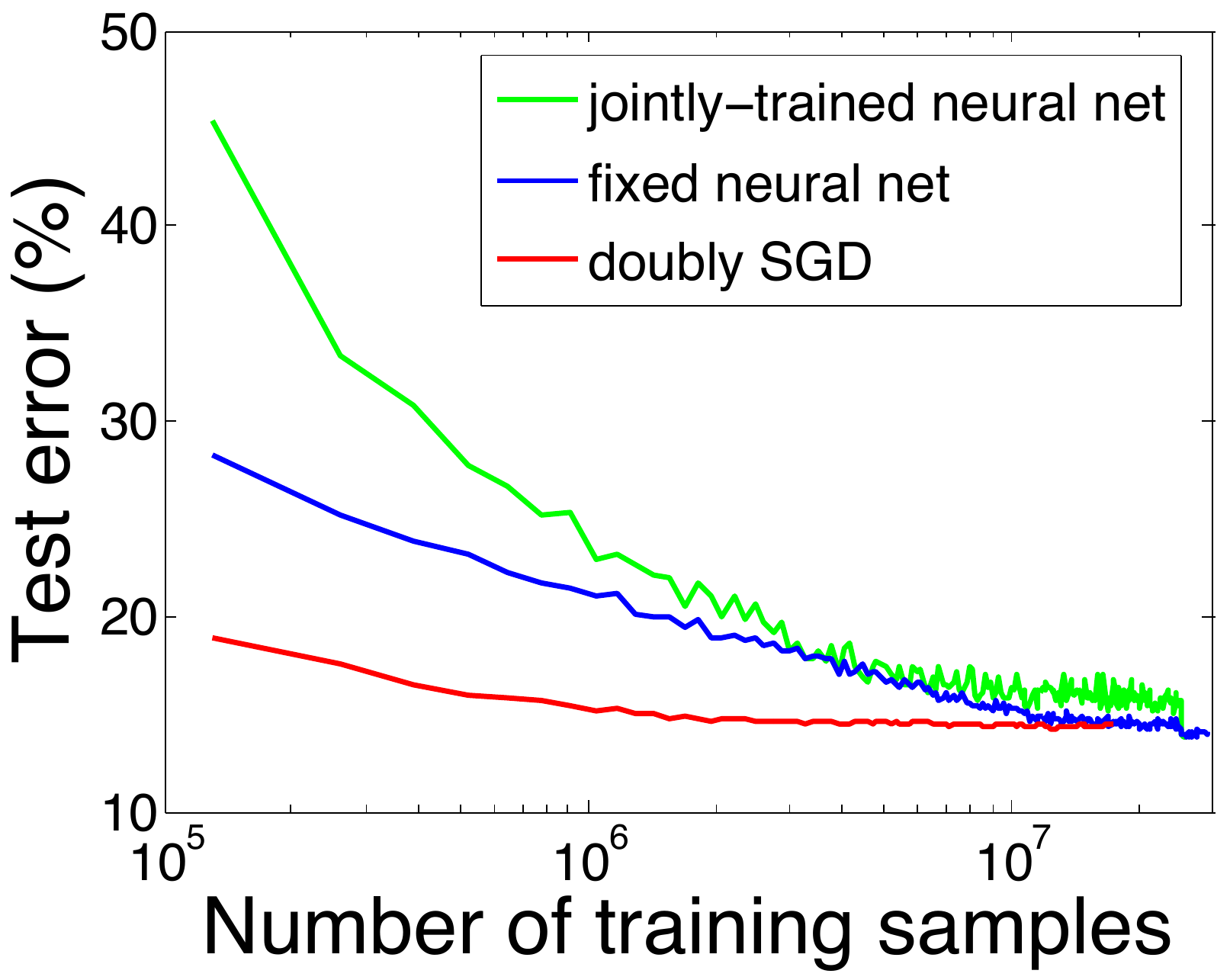} &
    \includegraphics[width=0.315\columnwidth]{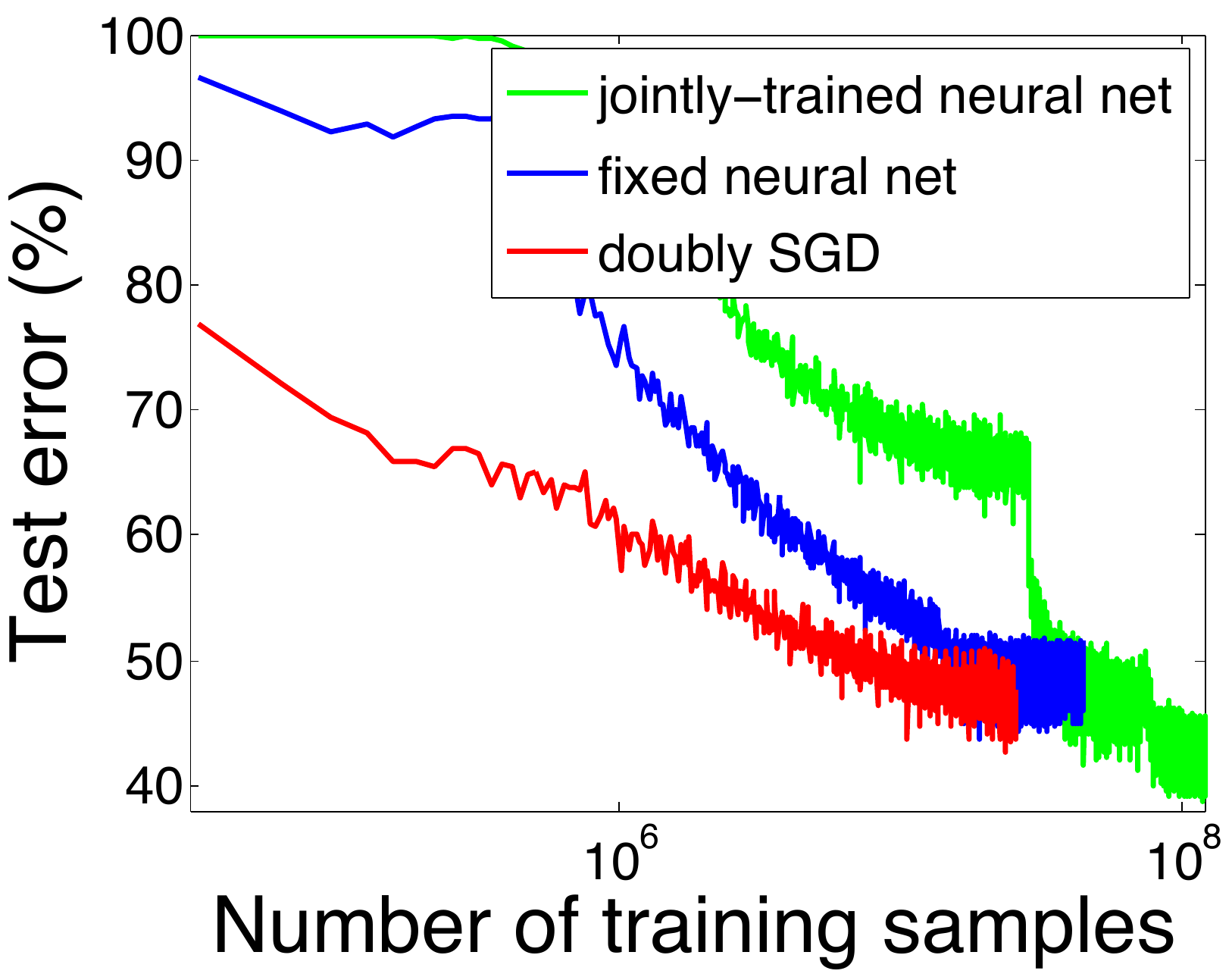} \\
     (1) MNIST 8M  &(2) CIFAR 10 & (3) ImageNet\\[3mm]
    \end{tabular}
  \begin{tabular}{cc}
    \includegraphics[width=0.315\textwidth]{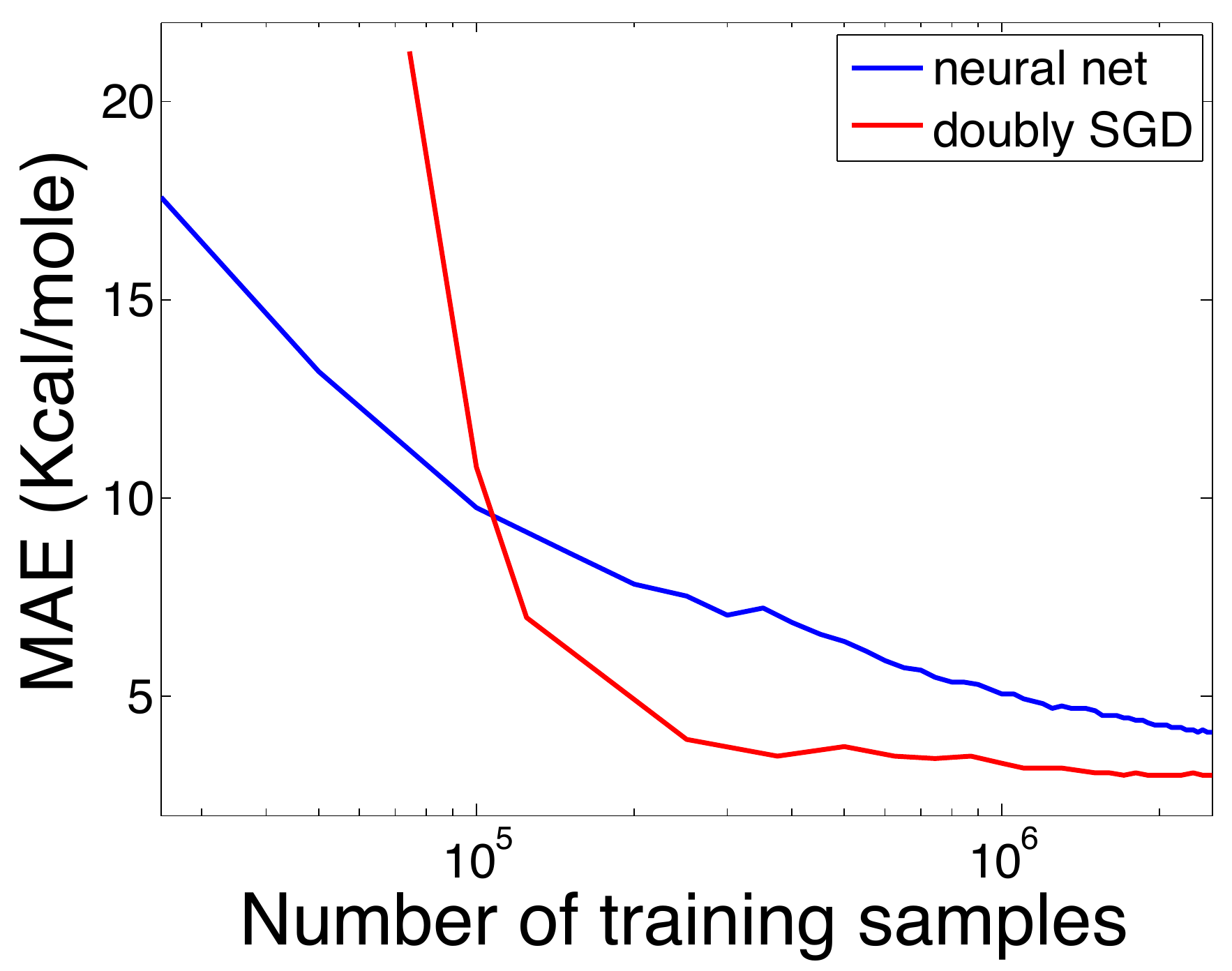} &
    \includegraphics[width=0.315\textwidth]{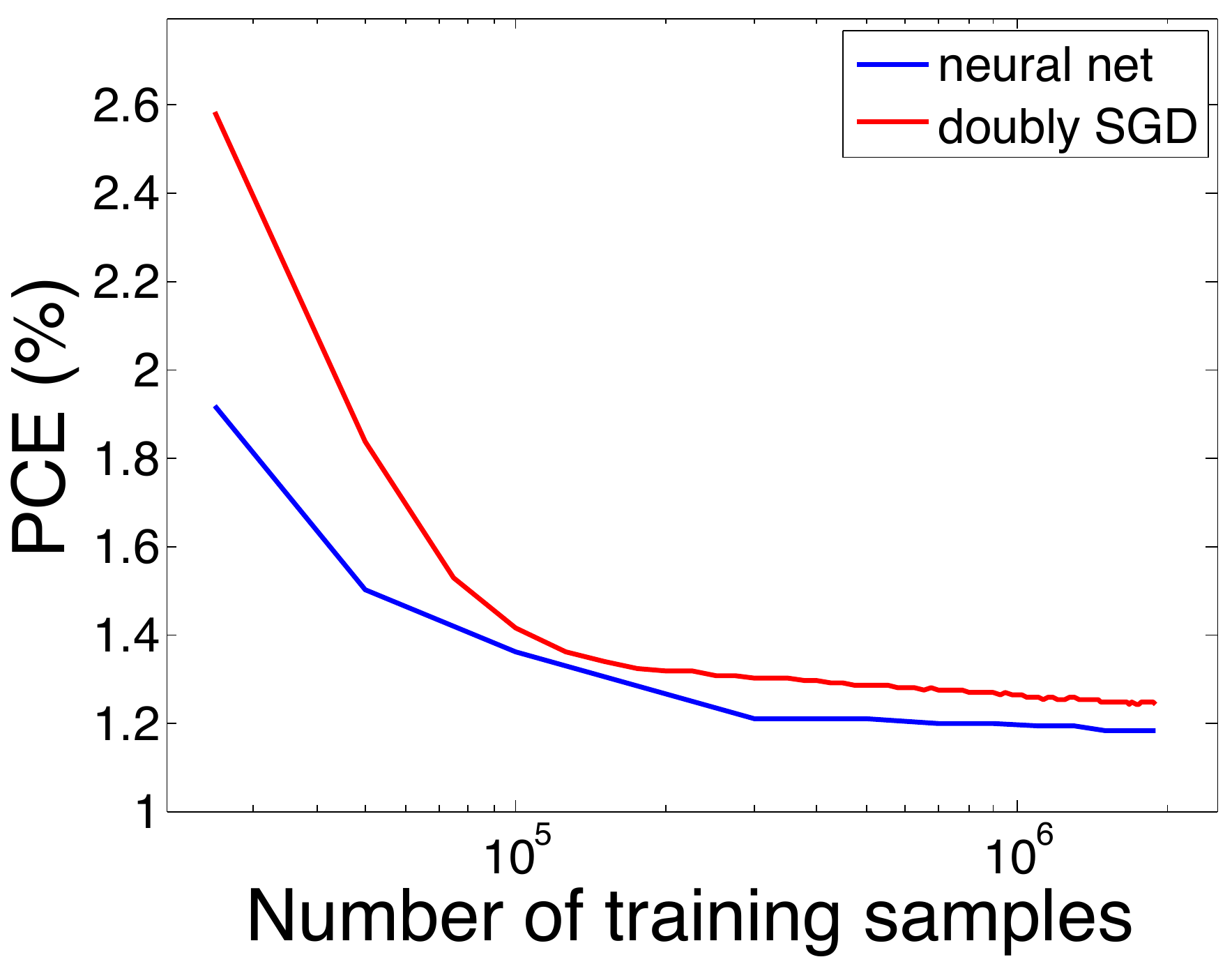} \\
  (4) QuantumMachine & (5) MolecularSpace.\\
  \end{tabular}
  \caption{Comparison with Neural Networks on datasets (6) -- (10).}
  \label{fig:results}}
\end{figure*}

We also compare our algorithm with the state-of-the-art neural network. In these experiments, the block size is set to be $O(10^4)$. Compared to the number of samples, $O(10^8)$, this block size is reasonable.

\paragraph{\bf MNIST 8M.} In this experiment, we compare to a variant of LeNet-5~\cite{LeCBotBenHaf98}, where all tanh units are replaced with rectified linear units. We also use more convolution filters and a larger fully connected layer. Specifically, the first two convolutions layers have 16 and 32 filters, respectively, and the fully connected layer contains 128 neurons. We use kernel logistic regression for the task. We extract features from the last max-pooling layer with dimension 1568, and use Gaussian RBF kernel with kernel bandwidth $\sigma$ equaling to four times the median pairwise distance. The regularization parameter $\nu$ is set to be $0.0005$.

 The result is shown in Figure~\ref{fig:results}(1). As expected, the neural net with pre-learned features is faster to train than the jointly-trained one. However, our method is much faster compared to both methods. In addition, it achieves a lower error rate (0.5\%) compared to the 0.6\% error provided by the neural nets.

\paragraph{\bf CIFAR 10.} In this experiment, we compare to a neural net with two convolution layers (after contrast normalization and max-pooling layers) and two local layers that achieves 11\% test error\footnote{The specification is at https://code.google.com/p/cuda-convnet/} on CIFAR 10~\cite{Krizhevsky09}. The features are extracted from the top max-pooling layer from a trained neural net with 2304 dimension. We use kernel logistic regression for this problem. The kernel bandwidth $\sigma$ for Gaussian RBF kernel is again four times the median pairwise distance. The regularization parameter $\nu$ is set to be $0.0005$. We also perform a PCA (without centering) to reduce the dimension to 256 before feeding to our method.

The result is shown in Figure~\ref{fig:results}(2). The test error for our method drops significantly faster in the earlier phase, then gradually converges to that achieved by the neural nets. Our method is able to produce the same performance within a much restricted time budget.

\paragraph{\bf ImageNet.} In this experiment, we compare our algorithm with the neural nets on the ImageNet 2012 dataset, which contains 1.3 million color images from 1000 classes. Each image is of size 256 $\times$ 256, and we randomly crop a 240 $\times$ 240 region with random horizontal flipping. The jointly-trained neural net is Alex-net~\cite{KriSutHin12}. The 9216 dimension features for our classifier and fixed neural net are from the last pooling layer of the jointly-trained neural net. The kernel bandwidth $\sigma$ for Gaussian RBF kernel is again four times the median pairwise distance. The regularization parameter $\nu$ is set to be $0.0005$.

Test error comparisons are shown in Figure~\ref{fig:results}(3). Our method achieves a test error of 44.5\% by further max-voting of 10 transformations of the test set while the jointly-trained neural net arrives at 42\% (without variations in color and illumination). At the same time, fixed neural net can only produce an error rate of 46\% with max-voting. There may be some advantages to train the network jointly such that the layers work together to achieve a better performance. Although there is still a gap to the best performance by the jointly-trained neural net, our method comes very close with much faster convergence rate. Moreover, it achieves superior performance than the neural net with pre-learned features, both in accuracy and speed.

\begin{figure*}[!t]
    \includegraphics[width=\columnwidth]{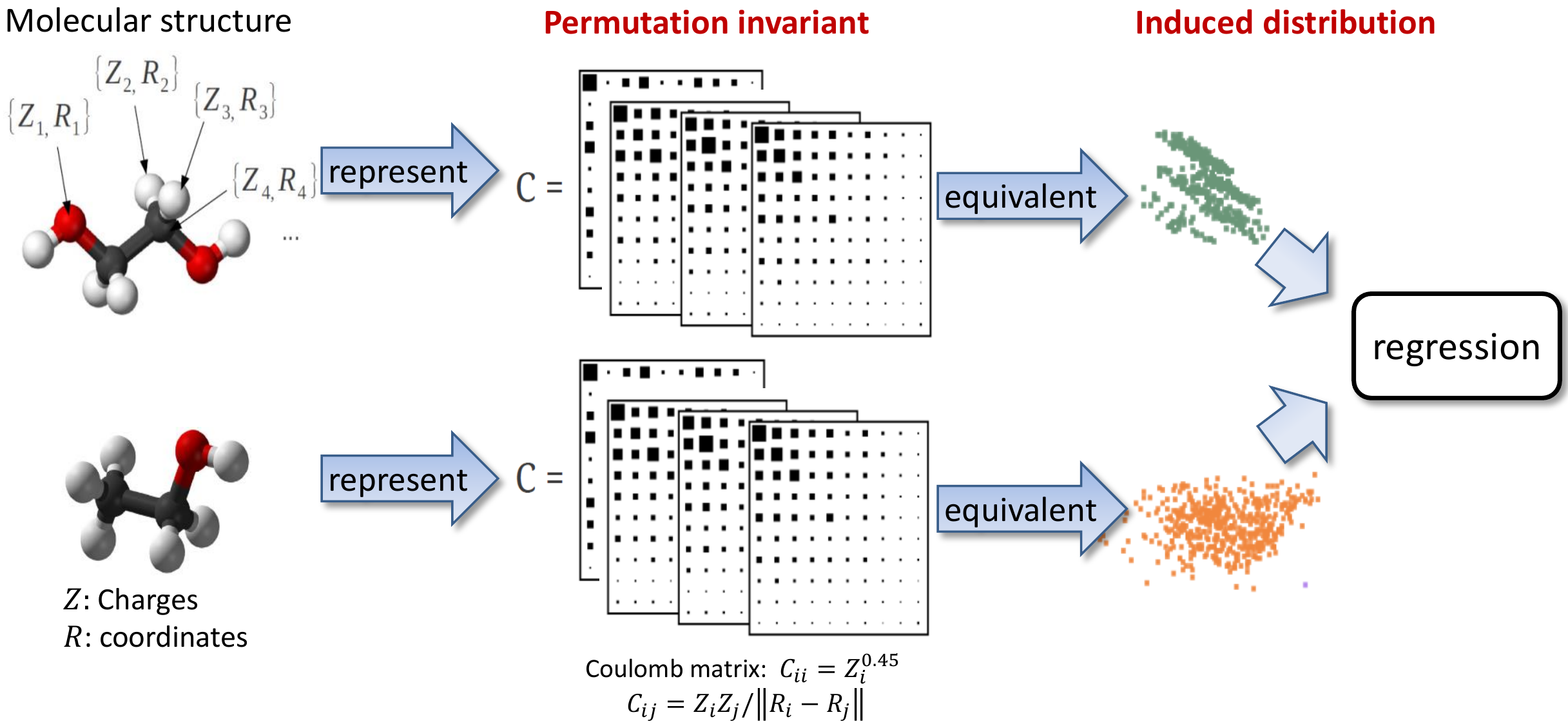}
  \caption{The computational procedure for predicting molecular property from molecular structure.}
  \label{fig:molecular_nn}
\end{figure*}

\subsection{Regression Comparisons to Neural Networks}

We test our algorithm for kernel ridge regression with neural network proposed in~\cite{MonHanFazRupetal12} on two large-scale real-world regression datasets, (9) and (10) in Table~\ref{table:datasets}. To our best knowledge, this is the first comparison between kernel ridge regression and neural network on the dataset MolecularSpace.

\paragraph{\bf QuantumMachine.} In this experiment, we use the same binary representations converted based on random Coulomb matrices as in~\cite{MonHanFazRupetal12}. We first generate a set of randomly sorted coulomb matrices for each molecule. And then, we break each dimension of the Coulomb matrix apart into steps and convert them to the binary predicates. Predictions are made by taking average of all prediction made on various Coulomb matrices of the same molecule. The procedure is illustrated in Figure.~\ref{fig:molecular_nn}. For this experiment, 40 sets of randomly permuted matrices are generated for each training example and 20 for each test example. We use Gaussian kernel with kernel bandwidth $\sigma = 60$ obtained by median trick. The batch size is set to be $50000$ and the feature block is $2^{11}$. The total dimension of random features is $2^{20}$.

The results are shown in Figure~\ref{fig:results}(4). In QuantumMachine dataset, our method achieves Mean Absolute Error (MAE) of $2.97$ kcal/mole, outperforming neural nets results, $3.51$ kcal/mole. Note that this result is already close to the $1$ kcal/mole required for chemical accuracy.

\paragraph{\bf MolecularSpace.} In this experiment, the task is to predict the power conversion efficiency (PCE) of the molecule. This dataset of 2.3 million molecular motifs is obtained from the Clean Energy Project Database. We use the same feature representation as for ``QuantumMachine'' dataset~\cite{MonHanFazRupetal12}. We set the kernel bandwidth of Gaussian RBF kernel to be $290$ by median trick. The batch size is set to be $25000$ and the feature block is $2^{11}$. The total dimension of random features is $2^{20}$.

The results are shown in Figure ~\ref{fig:results}(5). It can be seen that our method is comparable with neural network on this 2.3 million dataset.

\section{Discussion}

Our work contributes towards making kernel methods scalable for large-scale datasets. Specifically, by introducing artificial randomness associated with kernels besides the random data samples, we propose doubly stochastic functional gradient for kernel machines which makes the kernel machines efficient in both computation and memory requirement. Our algorithm successfully reduces the memory requirement of kernel machines from $O(dn)$ to $O(n)$. Meanwhile, we also show that our algorithm achieves the optimal rate of convergence, $O(1/t)$, for strongly convex stochastic optimization. We compare our algorithm on both classification and regression problems with the state-of-the-art neural networks as well as some other competing algorithms for kernel methods on several large-scale datasets. With our efficient algorithm, kernel methods could perform comparable to sophisticated-designed neural network empirically.

The theoretical analysis, which provides the rate of convergence \emph{independent} to the dimension, is also highly non-trivial. It twists martingale techniques and the vanilla analysis for stochastic gradient descent and provides a new perspective for analyzing optimization in infinite-dimensional spaces, which could be of independent interest. It should be pointed out that although we applied the algorithm to many kernel machines even with non-smooth loss functions, our current proof relies on the Lipschitz smoothness of the loss function. Extending the guarantee to non-smooth loss function will be one interesting future work.

Another key property of our method is its simplicity and ease of implementation which makes it versatile and easy to be extened in various aspects. It is straightforward to replace the sampling strategy for random features with Fastfood~\cite{LeSarSmo13} which enjoys the efficient computational cost, or Quasi-Monte Carlo sampling~\cite{YanSinAvrMah14}, data-dependent sampling~\cite{Bach15} which enjoys faster convergence rate with fewer generated features. Meanwhile, by \emph{back-propogation} trick, we could refine the random features by adapting their weights for better performance~\cite{YanMocDenFreetal14}. 

\section*{Acknowledgement}
M.B. is supported in part by NSF grant CCF-1101283, AFOSR grant FA9550-09-1-0538, a Microsoft Faculty Fellowship, and a Raytheon Faculty Fellowship. L.S. is supported in part by NSF IIS-1116886, NSF/NIH BIGDATA 1R01GM108341, NSF CAREER IIS-1350983, and a Raytheon Faculty Fellowship.

\bibliographystyle{unsrt}

\clearpage
\newpage

\begin{center}
{\Large Appendix}
\end{center}

\appendix

\section{Convergence Rate}\label{appendix:proof_details}

We first provide specific bounds and detailed proofs for the two error terms appeared in Theorem \ref{thm:expectation} and Theorem \ref{thm:probability}.

\subsection{Error due to random features}\label{appendix:error1}

\begin{lemma}\label{lem:random_feature} We have
\begin{enumerate}[label={(\roman*)}]
\item  For any $x\in \Xcal$,  $\EE_{\Dcal^t,\omegab^t}[|f_{t+1}(x) - h_{t+1}(x)|^2]\leqslant B^2_{1,t+1}:=4M^2(\kappa+\phi)^2\sum_{i=1}^t|a_t^i|^2.$

\item  For any $x\in \Xcal$, with probability at least $1- \delta$ over $(\Dcal^t,\omegab^t)$,
$$
  |f_{t+1}(x) - h_{t+1}(x)|^2 \leqslant
B^2_{2,t+1} := 2M^2(\kappa + \phi)^2\ln\rbr{\frac{2}{\delta}} \sum_{i=1}^t |a^i_t|^2
$$
\end{enumerate}
\end{lemma}

\begin{proof}
Let $V_i(x) =V_i(x;\Dcal^i,\omegab^i):=a_t^i \rbr{\zeta_i(x) - \xi_i(x)}$. Since $V_i(x)$ is a function of $(\Dcal^i,\omegab^i)$ and
$$
\EE_{\Dcal^i,\omegab^{i}}\sbr{V_i(x)| \omegab^{i-1}} = a_t^i \EE_{\Dcal^i,\omegab^{i} }\sbr{\zeta_{i}(x) - \xi_{i}(x) | \omegab^{i-1}} =a_t^i \EE_{\Dcal^i,\omegab^{i-1}}\sbr{\EE_{\omega^{i}}\sbr{\zeta_{i}(x) - \xi_{i}(x) | \omegab^{i-1}}}=0 ,
$$
we have that $\cbr{V_i(x)}$ is a martingal difference sequence. Further note that
$$
  |V_i(x)| \leqslant c_i= 2M(\phi+\kappa) |a_t^i|.
$$
Then by Azuma's Inequality, for any $\epsilon>0$,
$$
\Pr_{\Dcal^t,\omegab^t} \cbr{|\sum_{i=1}^t V_i(x)| \ge \epsilon}\leq2\exp\left\{-\frac{2\epsilon^2}{\sum_{i=1}^tc_i^2}\right\}
$$
which is equivalent as
$$
  \Pr_{\Dcal^t,\omegab^t} \cbr{\rbr{\sum_{i=1}^t V_i(x)}^2 \ge \ln(2/\delta)\sum_{i=1}^tc_i^2/2 }  \leqslant \delta.
$$
 Moreover,
\begin{align*}
  \EE_{\Dcal^t,\omegab^t} \sbr{\rbr{\sum_{i=1}^t V_i(x)}^2 } &=\int_0^\infty\Pr_{\Dcal^t,\omegab^t} \cbr{\rbr{\sum_{i=1}^t V_i(x)}^2 \ge \epsilon} d\epsilon = \int_0^\infty 2\exp\left\{-\frac{2\epsilon}{\sum_{i=1}^tc_i^2}\right\}d\epsilon =\sum_{i=1}^t c_i^2
\end{align*}
Since $f_{t+1}(x) - h_{t+1}(x) = \sum_{i=1}^t V_i(x)$, we immediately obtain the two parts of the lemma.
\end{proof}

\begin{lemma}\label{lem:coeff}
Suppose $\gamma_i = \frac{\theta}{i} (1\le i \leqslant t)$ and $\theta \nu \in (1,2) \cup \ZZ_+$. Then we have
\begin{enumerate}[label={(\arabic*)}]
\item $|a^i_t| \leqslant \frac{\theta}{t}$. Consequently, $\sum_{i=1}^t (a^i_t)^2 \leqslant \frac{\theta^2}{t}.$
\item $\sum_{i=1}^t\gamma_i |a_t^i|\leqslant \left\{
\begin{array}{ll}\frac{\theta^2 (\ln(t)+1)}{t}, &\text{if } \theta\nu\in [1,2),\\
                         \frac{\theta^2}{t}, &\text{if } \theta\nu\in [2,+\infty)\cap \ZZ_+\end{array}\right.$.
\end{enumerate}
\end{lemma}
\begin{proof}
$(1)$  follows by induction on $i$. $|a_t^t| \leqslant \frac \theta t$ is trivially true. We have
\begin{align*}
|a_t^i|
& = |a_t^{i+1} \frac{\gamma_i}{\gamma_{i+1}} (1-\nu \gamma_{i+1})|  = \frac{i+1}{i} |1-\frac{\nu \theta}{i+1}|\cdot |a_t^{i+1}| = |\frac{i+1 - \nu\theta}{i}|\cdot|a_t^{i+1}|.
\end{align*}
When $\nu\theta \in (1,2)$, $i - 1 < i+1 - \nu\theta < i$ for any $i \geqslant 1$, so $|a_t^i|  < |a_t^{i+1}| \leqslant \frac{\theta}{t}$.
When $\nu\theta \in \ZZ_+$, if $i > \nu \theta -1$, then $|a_t^i|  < |a_t^{i+1}| \leqslant \frac{\theta}{t}$; if $i \leqslant \nu \theta -1$, then $|a_t^i| =0$. For $(2)$, when $\theta\nu\in[1,2)$,
$$\sum_{i=1}^t\gamma_t|a_t^i|=\sum_{i=1}^t\frac{\theta^2}{i^2}\cdot\frac{i+1-\theta\nu}{i+1}\cdots\frac{t-\theta\nu}{t}\leqslant \sum_{i=1}^t\frac{\theta^2}{i^2}\cdot\frac{i}{i+1}\cdots\frac{t-1}{t}\leqslant\sum_{i=1}^t\frac{\theta^2}{it}\leqslant\frac{\theta^2(\ln(t)+1)}{t}.$$ When $\theta\nu\in\ZZ_+$ and $2\leq \theta\nu\leq t,$
$$\sum_{i=1}^t\gamma_t|a_t^i|=\sum_{i=2}^t\frac{\theta^2}{i^2}\cdot\frac{i+1-\theta\nu}{i+1}\cdots\frac{t-\theta\nu}{t}\leqslant \sum_{i=1}^t\frac{\theta^2}{i^2}\cdot\frac{i-1}{i+1}\cdots\frac{t-2}{t}\leqslant\sum_{i=2}^t\frac{\theta^2(i-1)}{it(t-1)}\leqslant\frac{\theta^2}{t}.$$
\end{proof}

\subsection{Error due to random data}\label{appendix:error2}
\begin{lemma}\label{lem:random_data}
Assume $l'(u,y)$ is $L$-Lipschitz continous in terms of $u\in\RR$. Let $f_*$ be the optimal solution to our target problem. Then
\begin{enumerate}[label={(\roman*)}]
\item If we set $\gamma_t=\frac{\theta}{t}$ with $\theta$ such that $\theta\nu\in(1,2)\cup\ZZ_+$, then
 $$\EE_{\Dcal^t,\omegab^t} \sbr{\nbr{h_{t+1} - f_\ast}^2_{\Hcal}}\leqslant \frac{Q_1^2}{t},$$
where
$$Q_1= \max\cbr{\nbr{f_\ast}_\Hcal, \frac{ Q_0 + \sqrt{ Q_0^2 + (2\theta\nu -1) (1+ \theta\nu)^2\theta^2  \kappa M^2 } }{2\nu \theta - 1}}, Q_0 =2\sqrt{2} \kappa^{1/2} (\kappa + \phi)LM \theta^2.
$$
Particularly, if $\theta\nu=1$, we have $Q_1\leq \max\cbr{\nbr{f_\ast}_\Hcal, 4\sqrt{2}((\kappa+\phi)L+\nu)\cdot\frac{\kappa^{1/2}M}{\nu^2}}$.

\item If we set $\gamma_t=\frac{\theta}{t}$ with $\theta$ such that $\theta\nu\in\ZZ_+$ and $t\ge \theta\nu$, then with probability at least $1-2\delta$ over $(\Dcal^t,\omegab^t)$,
$$\nbr{h_{t+1} - f_\ast}^2_{\Hcal}\leqslant Q_2^2\frac{\ln(2t/\delta)\ln(t)}{t}.$$
where
$$
Q_2= \max\cbr{\nbr{f_\ast}_\Hcal, Q_0 + \sqrt{ Q_0^2 + \kappa M^2(1+\theta\nu)^2(\theta^2+16\theta/\nu) }},
 Q_0 =4\sqrt{2}\kappa^{1/2}M\theta(8+(\kappa+\phi)\theta L).
$$
Particularly, if $\theta\nu=1$, we have $Q_2\leq  \max\cbr{\nbr{f_\ast}_\Hcal, 8\sqrt{2}((\kappa+\phi)L+9\nu)\cdot\frac{\kappa^{1/2}M}{\nu^2}}$.
\end{enumerate}
\end{lemma}

\begin{proof}
 For the sake of simple notations, let us first denote the following three different gradient terms, which are
\begin{align*}
g_t&=\xi_t+\nu h_t=l'(f_t(x_t),y_t)k(x_t,\cdot) +\nu h_t,\\
\hat g_t&=\hat\xi_t+\nu h_t=l'(h_t(x_t),y_t)k(x_t,\cdot) +\nu h_t,\\
\bar g_t&=\EE_{\Dcal_t}\sbr{\hat g_t}=\EE_{\Dcal_t}\sbr{l'(h_t(x_t),y_t)k(x_t,\cdot)} +\nu h_t.
\end{align*}
Note that by our previous definition, we have $h_{t+1}=h_t-\gamma_t g_t,\forall t\geq 1$.\\

Denote $A_t  = \nbr{h_t - f_\ast}^2_{\Hcal}$. Then we have
\begin{eqnarray*}
A_{t+1}
&=& \nbr{h_t - f_\ast - \gamma_t g_t}_{\Hcal}^2 \nonumber\\
&=& A_t +\gamma_t^2\nbr{g_t}^2_\Hcal - 2\gamma_t \langle h_t - f_\ast, g_t \rangle_\Hcal \nonumber\\
&=& A_t +\gamma_t^2\nbr{g_t}^2_\Hcal - 2\gamma_t \langle h_t - f_\ast, \bar g_t\rangle_\Hcal+ 2\gamma_t \langle h_t - f_\ast, \bar g_t-\hat g_t\rangle_\Hcal  + 2\gamma_t \langle h_t - f_\ast, \hat g_t -g_t  \rangle_\Hcal
\end{eqnarray*}
Because of the strongly convexity of~(\ref{eq:primal}) and optimality condition, we have
\begin{eqnarray*}
\langle h_t - f_\ast, \bar g_t\rangle_\Hcal  \ge \nu  \nbr{h_t - f_\ast}^2_{\Hcal}
\end{eqnarray*}
Hence, we have
\begin{equation}
\label{eqn:rec}
A_{t+1}\leqslant (1-2\gamma_t\nu)A_t+\gamma_t^2\nbr{g_t}^2_\Hcal + 2\gamma_t \langle h_t - f_\ast, \bar g_t-\hat g_t\rangle_\Hcal  + 2\gamma_t \langle h_t - f_\ast, \hat g_t -g_t  \rangle_\Hcal, \forall t\geq 1
\end{equation}

\textit{Proof for $(i)$}: Let us denote $\Mcal_t=\nbr{g_t}^2_\Hcal $, $\Ncal_t = \langle h_t - f_\ast, \bar g_t-\hat g_t\rangle_\Hcal $, $\Rcal_t=\langle h_t - f_\ast, \hat g_t -g_t  \rangle_\Hcal$. We first show that $\Mcal_t,\Ncal_t,\Rcal_t$ are bounded. Specifically, we have for $t\geq 1$,
\begin{enumerate}[label={(\arabic*)}]
\item $\Mcal_t\leq \kappa M^2(1+\nu c_t)^2$, where $c_t=\sqrt{\sum_{i,j=1}^{t-1}|a_{t-1}^i|\cdot|a_{t-1}^j|}$ for $t\geq 2$ and $c_1=0$;
\item $\EE_{\Dcal^t,\omegab^t}[\Ncal_t]=0$;
\item$\EE_{\Dcal^t,\omegab^t}[\Rcal_t]\leqslant \kappa^{1/2}LB_{1,t}\sqrt{\EE_{\Dcal^{t-1},\omegab^{t-1}}[A_t]}$, where $B^2_{1,t}:=4M^2(\kappa+\phi)^2\sum_{i=1}^{t-1}|a_{t-1}^i|^2$ for $t\geq 2$ and $B_{1,1}=0$;
\end{enumerate}
We prove these results separately in Lemma \ref{lem:bounds} below. Let us denote $e_t=\EE_{\Dcal^{t-1},\omegab^{t-1}}[A_t]$, given the above bounds, we arrive at the following recursion,
\begin{equation}
\label{eqn: rec2}
\begin{array}{c}
e_{t+1}\leqslant (1-2\gamma_t\nu)e_t+\kappa M^2\gamma_t^2(1+\nu c_t)^2+2\kappa^{1/2}L\gamma_t B_{1,t}\sqrt{e_t}.
\end{array}
\end{equation}

When $\gamma_t = \theta/t$ with $\theta$ such that $\theta\nu\in (1,2)\cup\ZZ_+$, from Lemma \ref{lem:coeff}, we have $|a_t^i|\leqslant\frac{\theta}{t},\forall 1\leq i\leq t$. Consequently, $c_t\leq \theta$ and $B^2_{1,t}\leqslant 4M^2(\kappa+\phi)\frac{\theta^2}{t-1}$. Applying these bounds leads to the refined recursion as follows
\begin{eqnarray*}
e_{t+1}\leqslant \rbr{1-  \frac{2\nu\theta}{t}} e_t +\kappa M^2\frac{\theta^2}{t^2}(1+\nu\theta)^2 +2\kappa^{1/2}L\frac{\theta}{t}\sqrt{4M^2(\kappa+\phi)^2\frac{\theta^2}{t-1}}\sqrt{e_t}
\end{eqnarray*}
that can be further written as
\begin{eqnarray*}
e_{t+1}\leqslant \rbr{1-  \frac{2\nu\theta}{t}} e_t  +\frac{\beta_1}{t}\sqrt{\frac{e_t}{t} }+\frac{ \beta_2}{t^2},
\end{eqnarray*}
where  $\beta_1=4\sqrt{2}\kappa^{1/2}LM(k+\phi)\theta^2$ and $\beta_2=\kappa M^2(1+\nu\theta)^2\theta^2$.
Invoking Lemma \ref{lem:rec} with  $\eta=2\theta\nu>1$, we obtain
$$e_t\leq \frac{Q_1^2}{t},$$
where
$Q_1= \max\cbr{\nbr{f_\ast}_\Hcal, \frac{ Q_0 + \sqrt{ Q_0^2 + (2\theta\nu -1) (1+ \theta\nu)^2\theta^2  \kappa M^2 } }{2\nu \theta - 1}}$,
 and  $Q_0 =2 \sqrt{2}\kappa^{1/2} (\kappa + \phi)LM \theta^2 .$

\textit{Proof for $(ii)$}:  Cumulating equations (\ref{eqn:rec}) with $i=1,\ldots t$, we end up with the following inequality
\begin{equation}
\label{eqn:cum}
\begin{array}{l}
A_{t+1}\leqslant \prod_{i=1}^t(1-2\gamma_i\nu)A_1 + 2\sum_{i=1}^t\gamma_i\prod_{j={i+1}}^t(1-2\nu\gamma_j) \langle h_i - f_\ast, \bar g_i-\hat g_i\rangle_\Hcal \\
\qquad\quad + 2\sum_{i=1}^t\gamma_i \prod_{j={i+1}}^t(1-2\nu\gamma_j)\langle h_i - f_\ast, \hat g_i -g_i  \rangle_\Hcal+\sum_{i=1}^t\gamma_i^2\prod_{j={i+1}}^t(1-2\nu\gamma_j)\nbr{g_i}^2_\Hcal
\end{array}
\end{equation}
Let us denote $b_t^i=\gamma_i \prod_{j={i+1}}^t(1-2\nu\gamma_j), 1\leq i\leq t$, the above inequality is equivalent as
\begin{eqnarray*}
A_{t+1}\leqslant \prod_{i=1}^t(1-2\gamma_i\nu)A_1 +\sum_{i=1}^t\gamma_ib_t^i\Mcal_i+2\sum_{i=1}^tb_t^i\Ncal_i+2\sum_{i=1}^tb_t^i\Rcal_i
\end{eqnarray*}
We first show that
\begin{enumerate}[label={(\arabic*)}]
\setcounter{enumi}{3}
\item for any $0<\delta<1/e$ and $t\geq 4$, with probability $1-\delta$ over $(\Dcal^t,\omegab^t) $,
$$\begin{array}{c}
\sum_{i=1}^t b_t^i\Ncal_i\leq 2\max\left\{4\kappa^{1/2}M\sqrt{\sum_{i=1}^t(b_t^i)^2A_i}, \,\,\max_{i}|b_t^i|\cdot C_0\sqrt{\ln(\ln(t)/\delta)}\right\}\sqrt{\ln(\ln(t)/\delta)},
\end{array}$$
 where $C_0=\frac{4\max_{1\leq i\leq t}\Mcal_i}{\nu}$.\\
\item for any $\delta>0$,  with probability $1-\delta$ over $(\Dcal^t,\omegab^t) $,
$$\begin{array}{c}\sum_{i=1}^tb_t^i\Rcal_i\leq \sum_{i=1}^tb_t^i\kappa^{1/2}L\hat B_{2,i}\sqrt{A_i},\end{array}$$
where $\hat B^2_{2,i}= 2M^2(\kappa + \phi)^2\ln\rbr{\frac{2t}{\delta}} \sum_{j=1}^{i-1} |a^j_{i-1}|^2$.
\end{enumerate}
Again, the proofs of these results are given separately in Lemma \ref{lem:bounds}.  Applying the above bounds leads to the refined recursion as follows,
\begin{eqnarray*}
A_{t+1}&\leqslant&\prod_{i=1}^t(1-2\gamma_i\nu)A_1 +\sum_{i=1}^t\gamma_ib_t^i\Mcal_i+2 \sum_{i=1}^tb_t^i\kappa^{1/2}LB_{2,i}\sqrt{A_i}\\
&\qquad& +4\max\left\{4\kappa^{1/2}M\sqrt{\sum_{i=1}^t(b_t^i)^2A_i},\, \max_{i}|b_t^i|\cdot C_0\sqrt{\ln(\ln(t)/\delta)}\right\}\sqrt{\ln(\ln(t)/\delta)}
\end{eqnarray*}
with probability $1-2\delta$.
When $\gamma_t = \theta/t$ with $\theta$ such that $\theta\nu\in\ZZ_+$, with similar reasons in Lemma \ref{lem:coeff}, we have $|b_t^i|\leqslant\frac{ \theta}{t}, 1\leq i\leq t$ and also we have $ \prod_{i=1}^t(1-2\gamma_i\nu) =  \prod_{i=1}^{\theta\nu-1}(1-2\frac{\theta\nu}{i}) \prod_{i=\theta\nu+1}^{t}(1-2\frac{\theta\nu}{i}) (1-2\frac{\theta\nu}{\theta\nu}) =0$, and $\sum_{i=1}^t\gamma_ib_t^i\leq \frac{\theta^2}{t}$. Therefore, we can rewrite the above recursion as
\begin{eqnarray}
A_{t+1}\leqslant \frac{\beta_1}{t}+\beta_2\sqrt{\ln(2t/\delta)}\cdot\sum_{i=1}^t\frac{\sqrt{A_i}}{t\sqrt{i}}+\beta_3\sqrt{\ln(\ln(t)/\delta)}\frac{\sqrt{\sum_{i=1}^tA_i}}{t}+\beta_4 \ln(\ln(t/\delta))\frac{1}{t}
\end{eqnarray}
where $\beta_1=\kappa M^2(1+\nu\theta)^2\theta^2$, $\beta_2=2\sqrt{2}\kappa^{1/2}LM(\kappa+\phi) \theta^2$,  $\beta_3=16\kappa^{1/2}M\theta$, $\beta_4=16\kappa M^2(1+\theta\nu)^2\theta/\nu$. Invoking Lemma \ref{lem:rec2}, we obtain
$$A_{t+1}\leqslant \frac{Q_2^2\ln(2t/\delta)\ln^2(t)}{t},$$
with the specified $Q_2$.
\end{proof}

\begin{lemma}\label{lem:bounds}
In this lemma, we prove the inequalities (1)--(5) in Lemma \ref{lem:random_data}.
\end{lemma}
\begin{proof}
Given the definitions of $\Mcal_t,\Ncal_t,\Rcal_t$ in Lemma \ref{lem:random_data}, we have
\begin{enumerate}[label={(\arabic*)}]
\item $\Mcal_t\leq \kappa M^2(1+\nu\sqrt{\sum_{i,j=1}^{t-1}|a_{t-1}^i|\cdot|a_{t-1}^j|})^2$;\\
This is because
$$\Mcal_t=\|g_t\|_\Hcal^2=\|\xi_t+\nu h_t\|_\Hcal^2\leqslant (\|\xi_t\|_\Hcal+\nu \|h_t\|_\Hcal)^2.$$
We have $$\|\xi_t\|_\Hcal=\|l'(f_t(x_t),y_t)k(x_t,\cdot)\|_\Hcal\leqslant \kappa^{1/2}M,$$ and

\begin{align*}
  \nbr{h_t }^2_\Hcal & = \sum_{i=1}^{t-1} \sum_{j=1}^{t-1} a_{t-1}^i a_{t-1}^j l'(f_i(x_i),y_i) l'(f_j(x_j),y_j) k(x_i, x_j) \\
  & \leqslant \kappa M^2\sum_{i=1}^{t-1} \sum_{j=1}^{t-1} |a_{t-1}^i|\cdot| a_{t-1}^j|.
\end{align*}

\item $\EE_{\Dcal^t,\omegab^t}[\Ncal_t]=0$;\\
This is because $\Ncal_t = \langle h_t - f_\ast, \bar g_t-\hat g_t\rangle_\Hcal $,
\begin{eqnarray*}
\EE_{\Dcal^t,\omegab^t}[\Ncal_t]&=&\EE_{\Dcal^{t-1},\omegab^t}\sbr{\EE_{D_t}\sbr{\langle h_t - f_\ast, \bar g_t-\hat g_t\rangle_\Hcal|\Dcal^{t-1},\omegab^t}}\\
&=&\EE_{\Dcal^{t-1},\omegab^t}\sbr{\langle h_t - f_\ast, \EE_{D_t}\sbr{\bar g_t-\hat g_t}\rangle_\Hcal}\\
&=&0.
\end{eqnarray*}\\

\item$\EE_{\Dcal^t,\omegab^t}[\Rcal_t]\leqslant \kappa^{1/2}LB_{1,t}\sqrt{\EE_{\Dcal^{t-1},\omegab^{t-1}}[A_t]}$, where $B^2_{1,t}:=4M^2(\kappa+\phi)^2\sum_{i=1}^{t-1}|a_{t-1}^i|^2;$\\
This is because $\Rcal_t=\langle h_t - f_\ast, \hat g_t -g_t  \rangle_\Hcal$,
\begin{eqnarray*}
\EE_{\Dcal^t,\omegab^t}[\Rcal_t]&=&\EE_{\Dcal^t,\omegab^t}\sbr{\langle h_t - f_\ast, \hat g_t -g_t  \rangle_\Hcal}\\
&=& \EE_{\Dcal^t,\omegab^t}\sbr{\langle h_t - f_\ast, [l'(f_t(x_t), y_t) - l'(h_t(x_t), y_t)]k(x_t, \cdot)  \rangle_\Hcal} \\
& \leqslant &\EE_{\Dcal^t,\omegab^t}\sbr{|l'(f_t(x_t), y_t) - l'(h_t(x_t), y_t)|\cdot \nbr{k(x_t, \cdot) }_\Hcal\cdot\nbr{h_t - f_\ast}_\Hcal }\\
&\leqslant&\kappa^{1/2}L\cdot\EE_{\Dcal^t,\omegab^t}\sbr{ |f_t(x_t) - h_t(x_t)| \nbr{h_t - f_\ast}_\Hcal  }\\
&\leqslant& \kappa^{1/2}L\sqrt{\EE_{\Dcal^t,\omegab^t}|f_t(x_t) - h_t(x_t)| ^2}\sqrt{\EE_{\Dcal^t,\omegab^t}\nbr{h_t - f_\ast}_\Hcal^2}\\
& \leqslant & \kappa^{1/2}LB_{1,t}\sqrt{\EE_{\Dcal^{t-1},\omegab^{t-1}}[A_t]}
\end{eqnarray*}
where the first and third inequalities are due to Cauchy--Schwarz Inequality and the second inequality is due to $L$-Lipschitz continuity of $l'(\cdot,\cdot)$ in the first parameter, and the last step  is due to Lemma~\ref{lem:random_feature} and the definition of $A_t$.\\

\item for any $0<\delta<1/e$ and $t\geqslant 4$, with probability at least $1-\delta$ over $(\Dcal^t,\omegab^t) $,
$$\begin{array}{c}
\sum_{i=1}^t b_t^i\Ncal_i\leqslant 2\max\left\{4\kappa^{1/2}M\sqrt{\sum_{i=1}^t(b_t^i)^2A_i}, \,\max_{i}|b_t^i|\cdot C_0\sqrt{\ln(\ln(t)/\delta)}\right\}\sqrt{\ln(\ln(t)/\delta)},
\end{array}$$
 where $C_0=\frac{4\max_{1\leq i\leq t}\Mcal_i}{\nu}$.\\
This result follows directly from Lemma 3 in~\cite{RakShaSri12}.  Let us define $d_i=d_i(\Dcal^i,\omegab^i):=b_t^i\Ncal_i=b_t^i\langle h_i - f_\ast, \bar g_i-\hat g_i\rangle_\Hcal , 1\leq i\leq t$, we have
\begin{itemize}
\item $\{d_i\}_{i=1}^t$ is martingale difference sequence since $\EE_{\Dcal^{i},\omegab^i}\sbr{\Ncal_i|\Dcal^{i-1},\omegab^{i-1}}=0$.
\item $|d_i|\leqslant  \max_i|b_t^i|\cdot C_0$, with $C_0=\frac{4\max_{1\leq i\leq t}\Mcal_i}{\nu}$, $\forall 1\leq i\leq t$.
\item $Var(d_i|\Dcal^{i-1},\omegab^{i-1})\leqslant 4\kappa M^2|b_t^i|^2A_i,\forall 1\leq i\leq t$.
\end{itemize}
Plugging in these specific bounds in Lemma 3 in [Alexander et.al., 2012], which is,
$$\begin{array}{c}
\Pr\left( \sum_{i=1}^t d_t\geq 2\max\{2\sigma_t, d_{max}\sqrt{\ln(1/\delta)}\}\sqrt{\ln(1/\delta)} \right)\leqslant \ln(t)\delta.\end{array}$$
where $\sigma_t^2=\sum_{i=1}^tVar_{i-1}(d_i)$ and $d_{max}=\max_{1\leq i\leq t}|d_i|$,  we immediately obtain the above inequality as desired.\\

\item for any $\delta>0$, with probability at least $1-\delta$ over $(\Dcal^t,\omegab^t) $,
$$\begin{array}{c}\sum_{i=1}^tb_t^i\Rcal_i\leq \sum_{i=1}^t|b_t^i|\kappa^{1/2}L\hat B_{2,i}\sqrt{A_i},\end{array}$$
where $\hat B^2_{2,i}= 2M^2(\kappa + \phi)^2\ln\rbr{\frac{2t}{\delta}} \sum_{j=1}^{i-1} |a^j_{i-1}|^2$.

This is because, for any $1\leq i\leq t$, recall that from analysis in (3), we have
$\Rcal_i\leq \kappa^{1/2}L|f_t(x_t)-h_t(x_t)|\cdot\|h_t-f_*\|_\Hcal$, therefore from Lemma \ref{lem:random_data},
$$\Pr(b_t^i\Rcal_i \leqslant \kappa^{1/2}L|b_t^i|\hat B_{2,i}\sqrt{A_i})\geqslant\Pr(|f_i(x_i)-h_i(x_i)|^2\leq \hat B^2_{2,i})\geqslant 1-\delta/t.$$
Taking the sum over $i$, we therefore get
$$\begin{array}{c}
\Pr(\sum_{i=1}^tb_t^i\Rcal_i\le  \sum_{i=1}^t|b_t^i|\kappa^{1/2}LB_{2,i}\sqrt{A_i})\geqslant 1-\delta.
\end{array}$$
\end{enumerate}
\end{proof}

Applying these lemmas immediately gives us Theorem \ref{thm:expectation} and Theorem \ref{thm:probability}, which implies pointwise distance between the solution $f_{t+1}(\cdot)$ and $f_*(\cdot)$. Now we prove similar bounds in the sense of  $L_\infty$ and  $L_2$ distance.

\section{\texorpdfstring{$L_\infty$}{L\infty} distance, \texorpdfstring{$L_2$}{L2} distance, and generalization bound}\label{appendix:L2}

\begin{corollary}[$L_\infty$ distance]\label{cor:Linf}
Theorem \ref{thm:expectation} also implies a bound in $L_{\infty}$ sense, namely,
 $$\EE_{\Dcal^t,\omegab^t}\nbr{f_{t+1}- f_\ast}_{\infty}^2\leq \frac{2C^2+2\kappa Q_1^2}{t}. $$
 Consequently, for the average solution $\hat f_{t+1}(\cdot):=\frac{1}{t}\sum_{i=1}^{t} f_{i}(\cdot)$, we also have
 $$\EE_{\Dcal^t,\omegab^t}\|\hat f_{t+1} - f_\ast\|_{\infty}^2\leq \frac{(2C^2+2\kappa Q_1^2)(\ln(t)+1)}{t}. $$
\end{corollary}
This is because $\nbr{f_{t+1} - f_\ast}_{\infty}=\max_{x\in \Xcal}|f_{t+1}(x) - f_\ast(x)|=|f_{t+1}(x_\ast) - f_\ast(x_\ast)|$, where $x_*\in\Xcal$ always exists since $\Xcal$ is closed and bounded. Note that the result for average solution can be improved without log factor using more sophisticated analysis (see also reference in \cite{RakShaSri12}).

\begin{corollary}[$L_2$ distance]\label{cor:L2}
With the choices of $\gamma_t$ in Lemma \ref{lem:random_data}, we have
\begin{enumerate}[label={(\roman*)}]
\item $\EE_{\Dcal^t,\omegab^t}\|f_{t+1}-f_*\|_2^2\leq \frac{2C^2+2\kappa Q_1^2}{t},$
\item  $\|f_{t+1}-f_*\|_2^2\leq\frac{C^2\ln(8\sqrt{e}t/\delta)+2\kappa Q_2^2\ln(2t/\delta)\ln^2(t)}{t},$
with probability at least $1-3\delta$ over $(\Dcal^t,\omegab^t)$.
\end{enumerate}
\end{corollary}
\begin{proof}\textit{(i)} follows directly from Theorem \ref{thm:expectation}. \textit{(ii)} can be proved as follows.
First, we have
$$\|f_{t+1}-f_*\|_2^2=\EE_{x}|f_{t+1}(x)-f_*(x)|^2\leq 2\EE_x|f_{t+1}(x)-h_{t+1}(x)|^2+2\kappa \|h_{t+1}-f_*\|_{\Hcal}.$$
 From Lemma \ref{lem:random_data}, with probability at least $1-2\delta$, we have
\begin{equation}\label{eqn:term2}
\|h_{t+1}-f_*\|_\Hcal^2\leq \frac{Q_2^2\ln(2t/\delta)\ln^2(t)}{t}.
\end{equation}
From Lemma \ref{lem:random_feature}, for any $x\in\Xcal$, we have
$$\Pr_{\Dcal^t,\omegab^t}\cbr{|f_{t+1}(x)-h_{t+1}(x)|^2\geq \frac{2(\kappa+\phi)^2M^2\ln(\frac{2}{\epsilon})\theta^2}{t}}\leqslant \epsilon.$$
Since $C^2=4(\kappa+\phi)^2 M^2\theta^2$, the above inequality can be writen as
$$\Pr_{\Dcal^t,\omegab^t}\cbr{|f_{t+1}(x)-h_{t+1}(x)|^2\geq \frac{C^2\ln(\frac{2}{\epsilon})}{2t}}\leqslant \epsilon.$$
which  leads to
$$\Pr_{x\sim \PP(x)}\Pr_{\Dcal^t,\omegab^t}\cbr{|f_{t+1}(x)-h_{t+1}(x)|^2\geq \frac{C^2\ln(\frac{2}{\epsilon})}{2t}}\leqslant \epsilon.$$
By Fubini's theorem and Markov's inequality, we have
$$\Pr_{\Dcal^t,\omegab^t}\cbr{\Pr_{x\sim \PP(x)}\cbr{|f_{t+1}(x)-h_{t+1}(x)|^2\geq \frac{C^2\ln(\frac{2}{\epsilon})}{2t}}\geqslant \frac{\epsilon}{\delta}}\leqslant \delta.$$
From the analysis in Lemma \ref{lem:random_feature}, we also have that $|f_{t+1}(x)-h_{t+1}(x)|\leqslant  C^2$. Therefore, with probability at least $1-\delta$ over $(\Dcal^t,\omegab^t)$, we have
$$\EE_{x\sim \PP(x)}[|f_{t+1}(x)-h_{t+1}(x)|^2]\leqslant \frac{C^2\ln(\frac{2}{\epsilon})}{2t}(1-\frac{\epsilon}{\delta}) +C^2\frac{\epsilon}{\delta}$$
Let $\epsilon =\frac{\delta}{4t}$, we have
\begin{equation}\label{eqn:term1}
\EE_{x\sim \PP(x)}[|f_{t+1}(x)-h_{t+1}(x)|^2]\leqslant \frac{C^2}{2t}(\ln(8t/\delta)+\frac{1}{2})=\frac{C^2\ln(8\sqrt{e}t/\delta)}{2t}.
\end{equation}
Summing up equation (\ref{eqn:term1}) and (\ref{eqn:term2}), we have
$$\|f_{t+1}-f_*\|_2^2\leq\frac{C^2\ln(8\sqrt{e}t/\delta)+2\kappa Q_2^2\ln(2t/\delta)\ln^2(t)}{t}$$
as desired.
\end{proof}

From the bound on $L_2$ distance, we can immediately get the generalization bound.\\
\noindent\textbf{Theorem}~\ref{thm:risk}~\textbf{(Generalization bound)}
{\it
Let the true risk be $R_{true}(f)=\EE_{(x,y)}\sbr{l(f(x),y)}$. Then with probability at least $1-3\delta$ over $(\Dcal^t,\omegab^t)$, and $C$ and $Q_2$ defined as previously
$$R_{true}(f_{t+1})-R_{true}(f_*)\leqslant \frac{(C\sqrt{\ln(8\sqrt{e}t/\delta)}+\sqrt{2\kappa}Q_2\sqrt{\ln(2t/\delta)}\ln(t))L}{\sqrt{t}}.$$
}
\begin{proof}
By the Lipschitz continuity of $l(\cdot,y)$ and Jensen's Inequality, we have
$$R_{true}(f_{t+1})-R_{true}(f_*)\leqslant L\EE_{x}|f_{t+1}(x)-f_*(x)|\leqslant L\sqrt{\EE_{x}|f_{t+1}(x)-f_*(x)|^2} = L\|f_{t+1}-f_*\|_2.$$
Then the theorem follows from Corollary~\ref{cor:L2}.
\end{proof}

\section{Suboptimality}\label{appendix:suboptimality}

For comprehensive purposes, we also provide the $O(1/t)$ bound for suboptimality.
\begin{corollary}\label{cor:suboptimality}
If we set $\gamma_t=\frac{\theta}{t}$ with $\theta\nu =1$, then the average solution $\hat f_{t+1}:=\frac{1}{t}\sum_{i=1}^{t}f_{i}$ satisfies
$$R(\EE_{\Dcal^t,\omegab^t}[\hat f_{t+1}])-R(f_*)\leqslant \frac{Q(\ln(t)+1)}{t}.$$
where $Q=(4\kappa M^2+2\sqrt{2}\kappa^{1/2}LM(\kappa+\phi)Q_1)/\nu$, with $Q_1$ defined as in Lemma \ref{lem:random_data}.
\end{corollary}

\begin{proof}
From the anallysis in Lemma \ref{lem:random_data},we have
$$\langle h_t-f_*,\bar g_t\rangle_\Hcal=\frac{1}{2\gamma_t}A_t-\frac{1}{2\gamma_t}A_{t+1}+\gamma_t \Mcal_t+\Ncal_t+\Rcal_t$$
Invoking strongly convexity of $R(f)$, we have $\langle h_t-f_*,\bar g_t\rangle\geq R(h_t)-R(f_*)+\frac{\nu}{2}\|h_t-f_*\|_\Hcal^2$. Taking expectaion on both size and use the bounds in last lemma, we have
$$\EE_{\Dcal^t,\omegab^t}[R(h_t)-R(f_*)]\leqslant(\frac{1}{2\gamma_t}-\frac{\nu}{2})e_t-\frac{1}{2\gamma_t}e_{t+1}+\gamma_t \kappa M^2(1+\nu c_t)^2+\kappa^{1/2}LB_{1,t}\sqrt{e_t}$$
Assume $\gamma_t=\frac{\theta}{t}$ with $\theta=\frac{1}{\nu}$, then cumulating the above inequalities  leads to
$$\sum_{i=1}^{t} \EE_{\Dcal^t,\omegab^t}[R(h_i)-R(f_*)]\leqslant \sum_{i=1}^{t}\gamma_i \kappa M^2(1+\nu c_i)^2+\sum_{i=1}^{t}\kappa^{1/2}LB_{1,i}\sqrt{e_i}$$
which can be further bounded by
\begin{eqnarray*}
\sum_{i=1}^{t} \EE_{\Dcal^t,\omegab^t}[R(h_i)-R(f_*)]&\leqslant& \sum_{i=1}^{t}\gamma_i \kappa M^2(1+\nu c_i)^2+\sum_{i=1}^{t}\kappa^{1/2}LB_{1,i}\sqrt{e_i}\\
&\leqslant& \frac{4\kappa M^2}{\nu}\sum_{i=1}^{t}\frac{1}{ i}+\frac{2\sqrt{2}\kappa^{1/2}LM(\kappa+\phi)}{\nu}\sum_{i=1}^{t}\sqrt{\frac{e_i}{i}}\\
&\leqslant& \frac{4\kappa M^2}{\nu}(\ln(t)+1)+\frac{2\sqrt{2}\kappa^{1/2}LM(\kappa+\phi)}{\nu} Q_1(\ln(t)+1)\\
&=&\frac{Q(\ln(t)+1)}{t}
\end{eqnarray*}
By convexity, we have $\EE_{\Dcal^t,\omegab^t}[R(\hat h_{t+1})-R(f_*)]\leqslant \frac{Q(\ln(t)+1)}{t}$.
The corollary then follows from the fact that $\EE_{\Dcal^t,\omegab^t}[\hat f_{t+1}]=\EE_{\Dcal^t,\omegab^t}[\hat h_{t+1}]$ and $R(\EE_{\Dcal^t,\omegab^t}[\hat h_{t+1}])\leqslant \EE_{\Dcal^t,\omegab^t}[R(\hat h_{t+1})]$.
\end{proof}

\subsection{Technical lemma for recursion bounds}
\begin{lemma}\label{lem:rec}
Suppose the sequence $\{\Gamma_t\}_{t=1}^{\infty}$ satisfies  $\Gamma_1\geq 0$, and $\forall t\geq 1$
$$
  \Gamma_{t+1}\leqslant \left (1-\frac{\eta}{t}\right)\Gamma_t+ \frac{\beta_1}{t \sqrt{t}}\sqrt{\Gamma_t} +  \frac{\beta_2}{t^2},
$$
where $\eta>1,\beta_1,\beta_2>0$. Then $\forall t \geqslant 1$,
$$
     \Gamma_t \leqslant \frac{R}{t},\text{~where~} R=\max\cbr{\Gamma_1, R_0^2}, R_0 = \frac{\beta_1  + \sqrt{\beta_1^2 + 4 (\eta-1) \beta_2 } }{2 (\eta -1)}.
$$
\end{lemma}
\begin{proof}
The proof follows by induction.
When $t=1$, it always holds true by the definition of $R$.  Assume the conclusion holds true for $t$ with $t\geq 1$, i.e., $\Gamma_t \leqslant \frac{R}{t}$, then we have
\begin{align*}
\Gamma_{t+1}
& \leqslant \left (1-\frac{\eta}{t}\right)\Gamma_t+  \frac{\beta_1}{t \sqrt{t}}\sqrt{\Gamma_t} + \frac{ \beta_2}{t^2}\\
& = \frac{R}{t} - \frac{ \eta R - \beta_1  \sqrt{R} - \beta_2  }{t^2} \leqslant \frac{R}{t+1} + \frac{R}{t(t+1)}  - \frac{ \eta R - \beta_1  \sqrt{R} - \beta_2 }{t^2} \\
& \leqslant \frac{R}{t+1} - \frac{1}{t^2} \sbr{ - R + \eta R - \beta_1  \sqrt{R} - \beta_2  }\leqslant \frac{R}{t+1}
\end{align*}
where the last step can be verified as follows.
\begin{align*}
(\eta-1) R - \beta_1  \sqrt{R} - \beta_2
& = (\eta-1) \sbr{ \sqrt{R} - \frac{\beta_1 }{2 (\eta-1) } }^2 - \frac{\beta_1^2 }{4(\eta-1)} - \beta_2  \\
& \geqslant (\eta-1) \sbr{ R_0  - \frac{\beta_1 }{2(\eta-1)} }^2 - \frac{\beta_1^2 }{4(\eta-1)} - \beta_2  \geqslant 0
\end{align*}
where the last step follows from the defintion of $R_0$.
\end{proof}

\begin{lemma}\label{lem:rec2}
Suppose the sequence $\{\Gamma_t\}_{t=1}^{\infty}$ satisfies
\begin{eqnarray*}
\Gamma_{t+1}\leqslant \frac{\beta_1}{t}+\beta_2\sqrt{\ln(2t/\delta)}\cdot\sum_{i=1}^t\frac{\sqrt{\Gamma_i}}{t\sqrt{i}}+\beta_3\sqrt{\ln(\ln(t)/\delta)}\frac{\sqrt{\sum_{i=1}^t\Gamma_i}}{t}+\beta_4 \ln(\ln(t/\delta))\frac{1}{t}
\end{eqnarray*}
where $\beta_1,\beta_2,\beta_3,\beta_4>0$ and $\delta\in(0,1/e)$. Then $\forall 1\leq j\leq t (t\geq 4)$,
$$\Gamma_{j}\leqslant \frac{R\ln(2t/\delta)\ln^2(t)}{j},\text{ where }  R=\max\{\Gamma_1,R_0^2\}, R_0=2\beta_2+2\sqrt{2}\beta_3+\sqrt{(2\beta_2+2\sqrt{2}\beta_3)^2+\beta_1+\beta_4}.$$
\end{lemma}
\begin{proof}
The proof follows by induction. When $j=1$ it is trivial. Let us assume it holds true for $1\leq j\leq t-1$, therefore,
\begin{eqnarray*}
\Gamma_{j+1}&\leqslant &\frac{\beta_1}{j}+\beta_2\sqrt{\ln(2j/\delta)}\cdot\sum_{i=1}^j\frac{\sqrt{\Gamma_i}}{j\sqrt{i}}+\beta_3\sqrt{\ln(\ln(j)/\delta)}\frac{\sqrt{\sum_{i=1}^j\Gamma_i}}{j}+\beta_4 \ln(\ln(j/\delta))\frac{1}{j}\\
&\leqslant& \frac{\beta_1}{j}+\beta_2\sqrt{\ln(2j/\delta)}/j\cdot\sum_{i=1}^j\frac{\sqrt{R\ln(2t/\delta)\ln^2(t)}}{i}\\
&&\qquad+\beta_3\sqrt{\ln(\ln(j)/\delta)}\frac{\sqrt{\sum_{i=1}^j R\ln(2t/\delta)\ln^2(t)/i}}{j}+\beta_4 \ln(\ln(j/\delta))\frac{1}{j}
\end{eqnarray*}
\begin{eqnarray*}
&\leqslant&  \frac{\beta_1}{j}+\beta_2\sqrt{\ln(2j/\delta)}/j \sqrt{R\ln(2t/\delta)\ln^2(t)}(1+\ln(j))\\
&&\qquad +\beta_3\sqrt{\ln(\ln(j)/\delta)}/j \sqrt{R\ln(2t/\delta)\ln^2(t)}\sqrt{\ln(j)+1}
+\beta_4\ln(\ln(j/\delta))\frac{1}{j}\\
&\leqslant& \frac{\beta_1}{j}+2\beta_2\sqrt{R}\ln(2t/\delta)\ln^2(t)/j+\sqrt{2}\beta_3\sqrt{R}\ln(2t/\delta)\ln^2(t)/j+\beta_4\ln(2t/\delta)\frac{1}{j}\\
&\leqslant& (2\beta_2+\sqrt{2}\beta_3)\sqrt{R}\frac{\ln(2t/\delta)\ln^2(t)}{j}+(\beta_1+\beta_4\ln(2t/\delta))\frac{1}{j}\\
&\leqslant &\frac{\ln(2t/\delta)\ln^2(t)}{j}[(2\beta_2+\sqrt{2}\beta_3)\sqrt{R}+\frac{\beta_1}{2}+\frac{\beta_4}{2})
\end{eqnarray*}
Since $\sqrt{R}\geqslant 2\beta_2+2\sqrt{2}\beta_3+\sqrt{(2\beta_2+2\sqrt{2}\beta_3)^2+\beta_1+\beta_4}$, we have
$(2\beta_2+2\sqrt{2}\beta_3)\sqrt{R}+\frac{\beta_1}{2}+\frac{\beta_4}{2}\leqslant R/2$. Hence, $\Gamma_{j+1}\leqslant  \frac{R\ln(2t/\delta)\ln^2(t)}{j+1}$.
\end{proof}

\section{Doubly Stochastic Gradient Algorithm for Posterior Variance Operator in Gaussian Process Regression}
\label{appendix:gp_update_rule}

As we show in Section~\ref{sec:doubly_sgd}, the estimation of the variance of the predictive distribution of Gaussian process for regression problem could be recast as estimating the operator $\Acal$ defined in~(\ref{eq:variance_operator}). We first demonstrate that the operator $\Acal$ is the solution to the following optimization problem
\begin{eqnarray*}
\min_{\Acal} R(\Acal) = \frac{1}{2n}\sum_{i=1}^n \|k(x_i, \cdot) - \Acal k(x_i, \cdot)\|_{\Hcal}^2 + \frac{\sigma^2}{2n}\|\Acal\|_{HS}^2
\end{eqnarray*}
where $\|\cdot\|_{HS}$ is the Hilbert-Schmidt norm of the operator. The gradient of $R(\Acal)$ with respect to $\Acal$ is
$$
\nabla R(\Acal) = \frac{1}{n}\sum_{i=1}^n\bigg((\Acal k(x, \cdot) - k(x, \cdot))\otimes k(x, \cdot)\bigg) + \frac{\sigma^2}{n}\Acal = \Acal\bigg(\Ccal + \frac{\sigma^2}{n}I\bigg) - \Ccal
$$
Set $\nabla R(\Acal) = 0$, we could obtain the optimal solution, $\Ccal\big(\Ccal + \frac{\sigma^2}{n} I\big)^{-1}$, exactly the same as (\ref{eq:variance_operator}).

To derive the doubly stochastic gradient update for $\Acal$, we start with stochastic functional gradient of $R(\Acal)$. Given $x_i\sim \PP(x)$, the stochastic functional gradient of $R(\Acal)$ is
$$
\psi(\cdot, \cdot) = \Acal\bigg(\widehat \Ccal + \frac{\sigma^2}{n}I\bigg) - \widehat \Ccal
$$
where $\widehat \Ccal = k(x_i, \cdot)\otimes k(x_i, \cdot)$ which leads to update 
\begin{eqnarray}\label{eq:gp_posterior_variance_update}
\Acal_{t+1} = \Acal_{t} - \gamma_t\psi = \bigg(1 - \frac{\sigma^2}{n}\gamma_t\bigg)\Acal_t - \gamma_t\bigg(\Acal_t\widehat\Ccal_t - \widehat\Ccal_t\bigg).
\end{eqnarray}

With such update rule, we could show that $\Acal_{t+1} = \sum_{i=1,j\ge i}^t\beta_{ij}^{t+1} k(x_i, \cdot)\otimes k(x_j, \cdot)$ by induction. Let $\Acal_1 = 0$, then, $\Acal_2 = \gamma_1 k(x_1, \cdot)\otimes k(x_1, \cdot)$. Assume at $t$-th iteration, $\Acal_t = \sum_{i=1,j\ge i}^{t-1}\beta_{ij}^{t} k(x_i, \cdot)\otimes k(x_j, \cdot)$, and notice that 
$$\Acal_t\widehat\Ccal_t = \Acal_t^\top(\cdot, x_t) \otimes k(x_t, \cdot)= \sum_{i=1}^{t-1}\bigg(\sum_{j\ge i}^{t-1} \beta_{ij}^tk(x_j, x_t)\bigg)k(x_i, \cdot)\otimes k(x_t, \cdot),
$$ 
we have $\Acal_{t+1} = \sum_{i=1,j\ge i}^{t}\beta_{ij}^{t+1} k(x_i, \cdot)\otimes k(x_j, \cdot)$ where
\begin{eqnarray*}
\beta_{ij}^{t+1} &=& \bigg(1 - \frac{\sigma^2}{n}\gamma_t\bigg)\beta_{ij}^{t}, \quad  \forall i\le j< t\\
\beta_{it}^{t+1} &=& -\gamma_t \sum_{j=1}^t\beta_{ij}^tk(x_j, x_t),\quad \forall i<t\\
\beta_{tt}^{t+1} &=& \gamma_t
\end{eqnarray*}

Recall
\begin{eqnarray*}
\widehat \Ccal_t = \EE_{\omega}[\phi_\omega(x_t)\phi_\omega(\cdot)]\otimes \EE_{\omega'}[\phi_{\omega'}(x_t)\phi_{\omega'}(\cdot)] = \EE_{\omega, \omega'}[\phi_\omega(x_t)\phi_{\omega'}(x_t)\phi_\omega(\cdot)\otimes \phi_{\omega'}(\cdot)],
\end{eqnarray*}
where $\omega, \omega'$ are independently sampled from $\PP(\omega)$, we could approximate the $\widehat \Ccal_t$ with random features, $\widehat\Ccal^{\omega, \omega'}_{t} = \phi_{\omega_t}(x_t)\phi_{\omega'_t}(x_t)\phi_{\omega_t}(\cdot)\otimes \phi_{\omega'_t}(\cdot)$. Plug random feature approximation into~(\ref{eq:gp_posterior_variance_update}) leads to 
\begin{eqnarray*}
\widehat \Acal_{t+1} = \bigg(1 - \frac{\sigma^2}{n}\gamma_t\bigg)\widehat \Acal_t - \gamma_t\bigg(\widehat\Acal_t^\top(\cdot, x_t)\otimes \phi_{\omega'_t}(x_t)\phi_{\omega'_t}(\cdot) - \widehat\Ccal^{\omega, \omega'}_{t}\bigg).
\end{eqnarray*}
Therefore, inductively, we could approximate $\Acal_{t+1}$ by 
$$
\widehat \Acal_{t+1} = \sum_{i\le j}^{t}\theta_{ij}^t\phi_{\omega_i}(\cdot)\otimes\phi_{\omega'_j}(\cdot)
$$
\begin{eqnarray*}
\theta_{ij} &=& \bigg(1 - \frac{\sigma^2}{n}\gamma_t\bigg)\theta_{ij},\, \forall i\le j<t\\
\theta_{it} &=& -\gamma_t\sum_{j\ge i}^{t-1}\theta_{ij}\phi_{\omega_j'}(x_t)\phi_{\omega_t'}(x_t),\, \forall i<t\\
\theta_{tt} &=& \gamma_t\phi_{\omega_t}(x_t)\phi_{\omega'_t}(x_t).
\end{eqnarray*}

\end{document}